\documentclass[12pt,]{article}
\usepackage{amsmath}
\usepackage{amssymb}
\usepackage{epsfig,amsthm}
\usepackage{amssymb,amsfonts}
\usepackage{dsfont}
\usepackage{graphicx}
\usepackage{indentfirst}
\usepackage{mathrsfs}
\usepackage{url}
\usepackage{bbm}
\usepackage{authblk}
\usepackage{longtable}
\usepackage{xparse}
\usepackage{lineno}
\usepackage{subfig}
\usepackage{bm}
\usepackage{bbm}
\usepackage{color}
\usepackage{mlmath}
\usepackage{xcolor}
\usepackage{tikz}

\usepackage{cleveref}
\newtheorem{thm}{Theorem}
\newtheorem{cor}{Corollary}
\newtheorem{defi}{Definition}
\newtheorem{lem}{Lemma}
\newtheorem{prop}{Proposition}
\newtheorem{rmk}{Remark}
\newtheorem{assump}{Assumption}
\newtheorem{exam}{Example}

\newenvironment{keywords}
{\bgroup\leftskip 20pt\rightskip 20pt \small\noindent{\bf Keywords:} }%
{\par\egroup\vskip 0.25ex}

 \newcommand*\patchAmsMathEnvironmentForLineno[1]{%
   \expandafter\let\csname old#1\expandafter\endcsname\csname #1\endcsname
   \expandafter\let\csname oldend#1\expandafter\endcsname\csname end#1\endcsname
   \renewenvironment{#1}%
      {\linenomath\csname old#1\endcsname}%
      {\csname oldend#1\endcsname\endlinenomath}}%
 \newcommand*\patchBothAmsMathEnvironmentsForLineno[1]{%
   \patchAmsMathEnvironmentForLineno{#1}%
   \patchAmsMathEnvironmentForLineno{#1*}}%
 \AtBeginDocument{%
 \patchBothAmsMathEnvironmentsForLineno{equation}%
 \patchBothAmsMathEnvironmentsForLineno{align}%
 \patchBothAmsMathEnvironmentsForLineno{flalign}%
 \patchBothAmsMathEnvironmentsForLineno{alignat}%
 \patchBothAmsMathEnvironmentsForLineno{gather}%
 \patchBothAmsMathEnvironmentsForLineno{multline}%
 }

\makeatletter
    \def\blfootnote{\xdef\@thefnmark{}\@footnotetext}
\makeatother\author[a]{Yuqing Li}

\affil[a]{School of Mathematical Sciences, Shanghai Jiao Tong University}
\author[b]{Tao Luo}
\affil[b]{School of Mathematical Sciences, Institute of Natural Sciences, MOE-LSC and Qing Yuan Research Institute, Shanghai Jiao Tong University }
\author[c]{Chao Ma}
\affil[c]{Department of Mathematics, Stanford University}

\date{\today}
 \title{Nonlinear Weighted Directed Acyclic Graph  and A Priori Estimates for Neural Networks}

\begin{document}
 
 \maketitle
 \allowdisplaybreaks

 \begin{abstract}
In an attempt to better understand structural  benefits and generalization power  of deep neural networks, we firstly present a novel graph theoretical formulation of neural network models, including fully connected, residual network~(ResNet) and densely connected networks~(DenseNet). Secondly, we extend the error analysis of  the population risk for two-layer network~\cite{ew2019prioriTwo} and ResNet~\cite{e2019prioriRes} to DenseNet, and show further that for neural networks satisfying certain mild conditions, similar estimates can be obtained. These estimates are a priori in nature since they  depend sorely on the information prior to the training process, in particular,  the bounds for the estimation errors do not suffer from the curse of dimensionality.

  \end{abstract}


\section{Introduction}
 
A central challenge in theoretical machine learning is to figure out  {the} source of the  generalization capabilities  of deep neural networks. Traditional statistical learning theory often  {fails} to provide satisfactory explanations~\cite{Bengio2010Understanding}. For this reason, there has been a flurry of recent papers  endeavor to  analyze the generalization error for neural networks~\cite{e2019prioriRes,ew2019prioriTwo,LiLiang2018SGDLearning,cao2020generalization,Daniely2017SGD,Zhu2019Learning,wojtowytsch2020priori,Neyshabur2017ExploreGeneralization,Barlett,Golowich2018Size,neyshabur2018PAC,neyshabur2018the,matengyu2019fisherrao}.  Since the problems that neural networks usually encounter tend to have very high dimensions, one issue of particular interest is   the \text{\em curse of dimensionality}~(CoD)~\cite{bellman2015adaptive}:  The computational cost depends exponentially on the dimension. 
However, in high dimensional settings, deep neural networks   have shown great promise in many   applications
and  do  not suffer from the CoD. Hence, we expect to obtain some proper  error estimates whose  error  bounds do not deteriorate as the input dimension grows.
In other words, an optimal error bound shall scale at a  rate  independent of the input dimension~\cite{ma2020towards}.
 
Another interesting  phenomenon is  that neural networks can be substantially deeper, more accurate, and efficient to train  if they contain shortcut connections from early layers to later layers. Most of the state-of-the-art neural networks  benefit from such bypassing paths~\cite{huang2017densely,He2016Deep,Highway,larsson2017fractalnet,Huang2016DeepwithStochastic}. For instance, the identity skip-connection blocks utilized in  ResNet~\cite{He2016Deep} serve as the  bypassing paths, and the counter-intuitive stochastic depth method introduced in~\cite{Huang2016DeepwithStochastic} shortens the effective depth in ResNet by randomly dropping layers during training. Theoretically, Hardt and Ma~\cite{Hardt2016identity} proved that for any residual linear networks with arbitrary depth, they possess no  spurious local optima. E et al.~\cite{e2019prioriRes} showed that optimal rate of the population risk can be guaranteed for ResNet,  {and for shallow neural networks, similar results still hold.  Venturi and Bruna \cite{Luca2019SpuriousTwoLayer} showed that spurious local minima can be avoided with high probability on overparametrized two-layer network models, E et al.~\cite{e2019prioriRes} also showed that optimal rate of the population risk can be guaranteed for two-layer network. However,
no result is available yet for deep networks without shortcut connections.}
Apart from them, the advantages of using shortcut connections remain to be discovered.

In this paper, we contribute to further understanding  of the above two aspects.
Firstly, we introduce a new representation for neural networks, namely the nonlinear weighted Directed Acyclic Graph~(DAG). The employment of DAG sheds light on the reasons behind success and failure of various network architectures from the perspective of linear algebra.  {Daniely~\cite{Daniely2017SGD}   also used DAGs to characterize the  neural network architectures but in a  different way from ours.
}
Using our representation,  revealing the  network structure  becomes a linear algebra problem. In particular, some typical feedforward neural networks such  as two-layer network, fully connected network, ResNet~\cite{e2019prioriRes,He2016Deep} and DenseNet~\cite{huang2017densely} can be represented by adjacency matrices.

In addition to representing commonly-used networks, we achieve the error bounds for a wide class of neural networks with DenseNet included.
For neural networks satisfying the assumption of shortcut connections~(\Cref{assump..BlockMatrix}),
 \text{\em a priori} estimates of the population risks can be established. According to~\cite{e2019prioriRes,ew2019prioriTwo}, most recent attempts~\cite{Neyshabur2017ExploreGeneralization,Barlett,Golowich2018Size,neyshabur2018PAC,neyshabur2018the,matengyu2019fisherrao} on bounding the generalization error of neural networks should be viewed as \text{\em a posteriori} estimates, in that the bounds rely on information acquired in the training process.
In comparison with the abovementioned  a posteriori  estimates, the a priori bound depends sorely on properties of the target function,
~hence it can be 
served as a  more natural reflection on   potential performances of different  neural networks.
The core of  our analysis is a specially designed parameter norm termed  the {\em weighted path norm}~\cite{e2019prioriRes,Neyshabur2017ExploreGeneralization} that proceeds by balancing between the complexity and the approximation. On one hand, the  weighted path norm gives control to  complexity of the hypothesis space induced by
neural networks~(\Cref{thm....RadmacherHQ}). On the other hand, the 
target functions   can be well approximated by neural networks, whose  weighted path  norm is  dominated by
the norm of target function, hence bringing about the a priori estimates~(\Cref{thm..ApproxError}). To sum up,  the hypothesis space determined by the norm is small enough to
have low complexity, but  also  large enough to have low approximation error. 

 {The organization of the paper is listed  as follows. In \Cref{Section...Prelim},  we give  some preliminary introduction to our problems. In \Cref{sec..NonlinearAdjacencyMatrix}, we propose our novel representation for feedforward neural networks. In \Cref{Section...Main Results}, we state our main results. In \Cref{Section...Proofs}, we give the full proof of the theorems and apply our estimates directly to DenseNet. Conclusions are
drawn in \Cref{Section...Conclusion}, and we compare our results with some related works.}

\section{Preliminaries}\label{Section...Prelim}
Throughout this paper, we use the following notations. We set $d$ as the input dimension,    and $n$ as the number of input samples. We  set $\Omega=[0,1]^d$ as the unit hypercube, and we let $\sigma(\cdot)$ be the Rectified Linear Unit ({ReLU}) activation, i.e., $\sigma(x)=\ReLU(x):=\max\{x,0\}$. Moreover, we use $\norm{\cdot}_1$ and $\norm{\cdot}_{\infty}$ to denote the $l_1$ and $l_{\infty}$ norms for vectors, and finally we use  $\norm{\cdot}_{1,1}$ to denote the entrywise $L_{1,1}$ norm for matrices. Specifically, for a matrix $\mA=[a_{i,j}]_{1\leq i \leq p, 1\leq j \leq q}$ of size $p\times q$, its entrywise $L_{1,1}$ norm reads
\begin{equation*}
   \norm{\mA}_{1,1}:=  \sum_{i=1}^p \sum_{j=1}^q |a_{ij}|.  
\end{equation*}
\subsection{Feedforward Neural Networks}\label{subsection....FeedForward}
In this section,  we firstly introduce some  commonly-used feedforward neural networks.
An artificial neural network~(NN) is a feedforward neural network when the connections of its nodes (neurons) do not form a cycle.  
 Some of the nodes are activated by a nonlinear function, and this function is termed the activation function.  {We use $\sigma(\cdot)$ to signify the ReLU activation, i.e., $\sigma(x)=\max
 \{0,x\}$. Thus $\sigma(\cdot)$ is $1$-Lipschitz because $\Abs{\sigma(z)-\sigma(z')} \leq C_\mathrm{L}\Abs{z-z'}$ with $C_\mathrm{L}=1$ for all $z,z'\in\sR$.} We denote the output function  of a neural network as $f(\vx;\vtheta)$,  where $\vx$ is a training sample, and  $\vtheta$ is the vector containing all   parameters of the function.  {We list out  some typical examples  of feedforward neural network and these examples will be studied  later from a different viewpoint. Since nodes of a feedforward neural network do not form a cycle, each network can be analogously treated as a weighted DAG~\cite{neyshabur2017thesis}, see the next section for more details.}

\noindent
{\bf 1. Two-layer Neural Network} 
\begin{equation}
    f_\mathrm{2Layer}(\vx;\vtheta)=\va^\T\sigma(\mW\vx),\label{eq..Architecture2Layer}
\end{equation}
where $\mW\in\sR^{m\times d}$, $\va\in\sR^{m}$, and $\vtheta=\mathrm{vec}\left\{\va,\mW\right\}$, where ``vec'' stands for the standard vectorization operation and it will be used hereafter.

\noindent
{\bf 2. Fully Connected Deep Network} 
\begin{equation}
  \left\{
    \begin{aligned}
      & \vh^{[0]} = \vx, \\
      & \vh^{[l]} = \sigma(\mW^{[l]} \vh^{[l-1]}),  \quad l=1,\cdots,L,\\
      & f_\mathrm{FC}(\vx;\vtheta) = \vu^\T \vh^{[L]},
    \end{aligned}\label{eq..ArchitectureFCNet}
  \right.
\end{equation}
where $\mW^{[l]}\in\sR^{m_l\times m_{l-1}}$, $m_0=d$, $\vu\in\sR^{m_L}$, and $\vtheta=\mathrm{vec}\left\{\{\mW^{[l]}\}_{l=1}^L,\vu\right\}$. 

\noindent
{\bf 3. Residual Network (ResNet)~\cite{e2019prioriRes,He2016Deep}} 
\begin{equation}
  \left\{
    \begin{aligned}
      & \vh^{[0]} = \mV\vx, \\
      & \vg^{[l]} = \sigma(\mW^{[l]} \vh^{[l-1]}),  \quad l=1,\cdots,L,\\
      & \vh^{[l]} = \vh^{[l-1]} + \mU^{[l]} \vg^{[l]}, \quad l=1,\cdots,L,\\
      & f_\mathrm{Res}(\vx;\vtheta) = \vu^\T \vh^{[L]},
    \end{aligned}\label{eq..ArchitectureResNet}
  \right.
\end{equation}
where $\mV\in\sR^{D\times d}$, $\mW^{[l]}\in\sR^{m\times D}$,
$\mU^{[l]}\in\sR^{D\times m}$, $\vu\in \sR^D$, $D\geq d+1$ and $\vtheta=\mathrm{vec}\left\{\mV,\{\mW^{[l]},\mU^{[l]}\}_{l=1}^L,\vu\right\}$.  

\noindent
{\bf 4. Dense Network (DenseNet)~\cite{huang2017densely}}
\begin{equation}
  \left\{
    \begin{aligned}
      & \vh^{[0]} = \mV\vx,\\
      & \vg^{[l]} = \sigma(\mW^{[l]}\vh^{[l-1]}),\quad l=1,\cdots,L,\\
      & \vh^{[l]} = \begin{pmatrix} \vh^{[l-1]}\\ \mU^{[l]}\vg^{[l]} \end{pmatrix},\quad l=1,\cdots,L,\\
      & f_\mathrm{Dense}(\vx;\vtheta) = \vu^\T\vh^{[L]},
    \end{aligned}\label{eq..ArchitectureDenseNet}
  \right.
\end{equation}
where $\mV\in\sR^{k_0\times d}$, $\mW^{[l]}\in\sR^{lm\times (k_0+(l-1)k)}$, $\vU^{[l]}\in\sR^{k\times lm}$, $\vu\in\sR^{k_0+Lk}$, $k_0\geq d+1$, and $\vtheta=\mathrm{vec}\left\{\mV,\{\mW^{[l]},\mU^{[l]}\}_{l=1}^L,\vu\right\}$.   For each $l$,
$\vh^{[l]}$ is the output  of  layer $l$, whose dimension is $k_0+lk$, where $k\geq 1$.  {Unlike ResNet, DenseNet uses concatenation instead of direct addition after ``going through'' the skip connection block. In particular, we} observe that the  dimension of $\vh^{[l]}$ grows linearly with respect to the number of layers, and we term $k$  the growth rate.
Usually, a relatively small growth rate (such as ten or twelve) is sufficient to obtain  state-of-the-art results on standard datasets, such as CIFAR-10 and ImageNet.
\begin{rmk}
 {In practice, every neuron has a bias and a layer is calculated as $\sigma(\widehat{\mW}\hat{\vx}+\vb)$, where  $\widehat{\mW}$ is the parameter matrix, $\hat{\vx}$ is the original data,  and $\vb$ is the bias vector. We point out that our formulations for Fully Connected Deep Network above can represent the counterpart with biased neurons by considering the extended data $\vx=(\hat{\vx},1)$ and parameter matrix \begin{equation*}
    \mW = \left[\begin{array}{cc}
        \widehat{\mW} & \vb \\
         0 & 1
    \end{array}\right].
\end{equation*}}

\end{rmk}
\subsection{Barron Space, Path Norm, and Rademacher Complexity}\label{Subsection....BarronSpacePathNorm}
Inspired by \cite{e2019prioriRes,ew2019prioriTwo}  and  {the} references therein,
we study a specific type of  target functions. Recall that $\Omega=[0,1]^d$ is the unit hypercube, and we consider target functions with domain $\Omega$.
\begin{defi}[Barron function and Barron space]\label{definition....Barron function and Barron space}
  A function $f:\Omega\to\sR$ is called a {\em Barron function} if $f$ admits  the following expectation representation$\mathrm{:}$
  \begin{equation}\label{Subsection...BarronSpace..Eq...FunctionRepresentation}
      f(\vx)=\Exp_{(a,\vw)\sim\rho}\left[a\sigma(\vw^\T\vx)\right],
  \end{equation}
  where $\rho$ is a probability distribution over $\sR^{d+1}$.
  
   For a Barron function, we define the {\em Barron norm} as
    {
    \begin{equation}\label{Subsection...BarronSpace..Eq.....BarronNorm}
      \Norm{f}_\fB:=\inf\limits_{\rho\in \fP_f} \Exp_{(a,\vw)\sim\rho}
      \abs{a}\norm{\vw}_1,
    \end{equation}
    }
    where 
    $
      \fP_f = \{\rho\mid f(\vx) = \Exp_{(a,\vw)\sim\rho}[a\sigma(\vw^\T\vx)]\}
    $.
    
  Equipped with the Barron Norm \eqref{Subsection...BarronSpace..Eq.....BarronNorm},  the {\em Barron space} $\fB$ is the set of Barron functions with finite Barron norm, i.e.,
    \begin{equation}\label{Subsection...BarronSpace..Eq..BarronSpace}
        \fB=\{f:\Omega\to\sR\mid \norm{f}_\fB<\infty\}.
    \end{equation}
\end{defi}
 {Normally, a Barron space contains functions with low complexity, such as sufficiently smooth functions in the Sobolev space $H^s(\sR^d)$ for $s>\frac{d}{2}+1$. It also contains non-smooth functions like those represented by two-layer neural networks.  Moreover, Barron space is strictly bigger than the Reproducing Kernel Hilbert Space (RKHS) induced by the Neural Tangent Kernel (NTK)~\cite{jacot2018neural}, and one may refer to~\cite{wojtowytsch2020kolmogorov,wojtowytsch2020representation} for detailed discussions. On the other hand, generalization of neural networks in the NTK regime has been studied in~\cite{arora2019fine,chen2020generalized}.}

For a feedforward neural network, we define   a parameter-based norm as an analog of the  path norm of two-layer neural networks~\cite{ew2019prioriTwo}, and the $l_1$ path norm of the residual networks~\cite{e2019prioriRes,Neyshabur2017ExploreGeneralization}. We term it the {\em weighted path norm}.  {In order to describe the norm, we shall introduce firstly the concept of {\em path}. To start with, a path is an ordered sequence of scalar operations in the computational process of  neural networks that originates from the input and ends at the output, which could also be  viewed as a connected chain of edges in the computation graph. For example, a path in two-layer neural networks contains one parameter from the input layer, one nonlinear activation, and one parameter from the output layer connecting this activation to the output. For deep neural networks, a path contains a collection of linear operations (including trainable parameters and fixed parameters) and  nonlinear counterparts. Let $\fP$ be a path, and  we   denote the number of linear operations in the path $\fP$ by $\mathrm{len}(\fP)$, and $\mathrm{nl}(\fP)$ for the number of nonlinearities   the path goes through. Finally, let $\{w_i^{\fP}\}_{i=1}^{\mathrm{len}(\fP)}$ be the parameters associated with linear operations throughout the path $\fP$, and we define the weighted path norm as follows.}

\begin{defi}[Weighted path norm]\label{Definition....Weighted Path Norm}
   Given a network  $f(\cdot;\vtheta)$, we define the {\em weighted path norm} of $f$ as 
   \begin{equation}\label{Subsection...BarronSpace..Eq..PathNorm}
       \Norm{f}_\mathrm{P}=\Norm{\vtheta}_\mathrm{P}
       =\sum_{\fP}3^{\textrm{nl}(\fP)}\prod_{l=1}^{\mathrm{len}(\fP)}\Abs{w_l^\fP}.
   \end{equation}
\end{defi} 
Heuristically speaking, the weighted path norm tends to take large account of the paths that undergoes more nonlinearities,  {i.e., by assigning bigger weights to paths going through more nonlinearities.}
 {The weight characterizes the increased complexity of the hypothesis space induced by nonlinearities. In particular, the weight factor $3^{\textrm{nl}(\fP)}$ in~\eqref{Subsection...BarronSpace..Eq..PathNorm} was first taken in~\cite{e2019prioriRes} on the a priori estimate for ResNet, it was taken for the convenience of analysis and may not be optimal. In the latest version of~\cite{e2019prioriRes}, the base number is reduced to $2$.}
 {By using the} \text {\em symbol} for the  adjacency matrix representation proposed in \Cref{subsection....NonlinearDAG}, we come up with a more handy-but-equivalent characterization for the weighted path norm in \Cref{prop..property}, a cornerstone upon on which some useful estimates are derived.

Finally, to bound the generalization gap, we recall the definition of Rademacher complexity.
\begin{defi}[Rademacher complexity]
  Given a family of functions $\fH$ and a set of samples $S=\{\vz_i\}_{i=1}^n$, the \emph{(empirical) Rademacher complexity} of $\fH$ with respect to $S$ is defined as
  \begin{equation}
    \Rad_S(\fH)=\frac{1}{n}\Exp_{\vtau}\left[\sup_{h\in\fH}\sum_{i=1}^n\tau_i h(\vz_i)\right],\label{eq..RadComplexityDefi}
  \end{equation}
  where the $\left\{\tau_i\right\}_{i=1}^n$ are i.i.d. random variables with $\Prob\{\tau_i=1\}=\Prob\{\tau_i=-1\}=\frac{1}{2}$.
\end{defi}

\section{Nonlinear Weighted DAG and Adjacency Matrix Representation}\label{sec..NonlinearAdjacencyMatrix}

In this section, we systematically present our novel representation for feedforward neural networks. We discuss several properties obtained from the incorporation of  this new representation in \Cref{subsection...properties},  and   some concrete examples are given out in \Cref{subsection....examples}, using the above-mentioned  networks in \Cref{subsection....FeedForward}.

\subsection{Adjacency Matrix Representation and Symbols for DAG}\label{subsection....NonlinearDAG}
  {We start this section by bringing out  the definitions of the directed graph and neural network.}
 {\begin{defi}[Directed graph]
    A \emph{directed graph} $G=(V,E)$ is an ordered pair of sets. Here $V$ is called the set of \emph{nodes} (or \emph{vertices} or \emph{neurons}), and $E\subset V\times V$ is called the set of \emph{edges} (or more precisely, \emph{directed edges}). For vertices $v_i,v_j\in V$, if $(v_j,v_i)\in E$, then the edge is denoted by $e_{i\leftarrow j}$ and said to be directed from the \emph{tail} $v_j$ to the \emph{head} $v_i$. A \emph{cycle} is a finite sequence of nodes $v_0,v_1,\cdots,v_k$ such that $v_0=v_k$ and $(v_i,v_{i+1})$ is an edge for all $i=0,1,\cdots,k-1$. A \emph{directed acyclic graph (DAG)} is a directed graph that has no cycles.

\end{defi}}

 {\begin{defi}[Neural network] \label{defi...NN}
    A \emph{neural network} $\mathrm{NN}$
    consists of an \emph{architecture} (a directed graph) $G=(V,E)$ with the partition of edges $E:=E_\mathrm{fix}\sqcup E_\mathrm{para}\sqcup E_\mathrm{non}$, \emph{a collection of fixed weights/parameters} (a real-valued function on $E_\mathrm{fix}$) $\gamma: E_\mathrm{fix}\to \sR$, \emph{a collection of trainable weights/parameters} (a real-valued function on $E_\mathrm{para}$) $\theta: E_\mathrm{para}\to \sR$, and  \emph{an activation} (usually, a nonlinear function) $\sigma(\cdot): \sR \to \sR$. Here $\sqcup$ is the disjoint union of sets. We  write $\mathrm{NN}(G,\vtheta,\vgamma,\sigma)$ to signify this neural network with
    \begin{align}
        \vgamma&:=\mathrm{vec}\{\gamma_{ij}:=\gamma(e_{i\leftarrow j})\mid e_{i\leftarrow j}\in E_\mathrm{fix}\},\label{eq..fixedWeight}\\
       \vtheta&:=\mathrm{vec}\{\theta_{ij}:=\theta(e_{i\leftarrow j})\mid e_{i\leftarrow j}\in E_\mathrm{para}\}.\label{eq..trainableWeight}
    \end{align} 
    We also define $w:E_\mathrm{fix}\sqcup E_\mathrm{para}\to \sR$ by setting $w(e_{i\leftarrow j}):=\gamma(e_{i\leftarrow j})$ on $E_\mathrm{fix}$ and $w(e_{i\leftarrow j}):=\theta(e_{i\leftarrow j})$ on $E_\mathrm{para}$. 
    We further set $N:=\#E$ as the number of edges, and  $N_\mathrm{fix}:=\#E_\mathrm{fix}$, $N_\mathrm{para}:=\#E_\mathrm{para}$, and $N_\mathrm{non}:=\#E_\mathrm{non}$. Obviously, $N=N_\mathrm{fix}+N_\mathrm{para}+N_\mathrm{non}$.
\end{defi}}

 {We remark that $\vgamma$ in \eqref{eq..fixedWeight} is pre-determined and fixed, while $\vtheta$ in \eqref{eq..trainableWeight} is trainable. For the existence results in main theorems, we refer to the existence of   $\vtheta$ after  the network architecture $G=(V,E)$, activation $\sigma$ and fixed weights $\vgamma$ are given~(See  \Cref{thm..ApproxError} and \Cref{thm..Apriori}).}

 {\begin{defi}[Feedforward neural network] 
    A \emph{feedforward neural network} is a neural network $\mathrm{NN}(G,\vtheta,\vgamma,\sigma)$ in which the graph $G=(V,E)$ contains no cycles. An \emph{input/source neuron} (or \emph{output/sink neuron}) is a vertex in $V$ that is not the head (or tail) of any edge in $E$. The set of input neurons and output neurons are denoted by $V_\mathrm{in}$ and $V_\mathrm{out}$  respectively. The vertices in $V_\mathrm{hid}=V\backslash (V_\mathrm{in}\cup V_\mathrm{out})$ are called the \emph{hidden neurons}. The input dimension and output dimension of the network are $d=\abs{V_\mathrm{in}}$ and $d'=\abs{V_\mathrm{out}}$.
\end{defi}}

\begin{figure}
 \centering
\includegraphics[width=\textwidth]{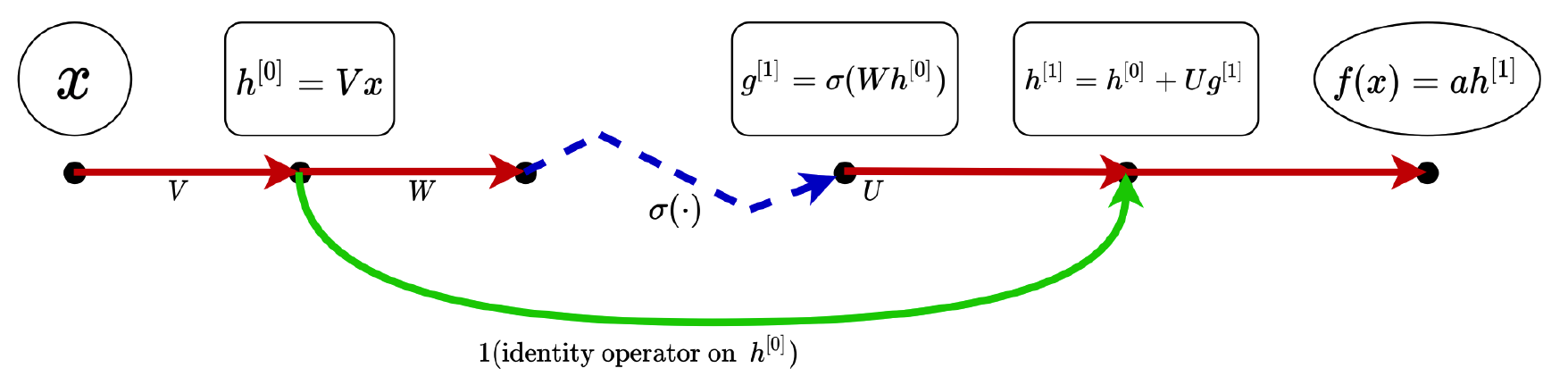}
    \caption{a particular nonlinear weighted DAG for ResNet }
    \label{fig....Example}
\end{figure}

 {Note that for all four examples given in \Cref{subsection....FeedForward}, $d'=1$.}

 {
\begin{defi}[Adjacency matrix representation]
    Given any feedforward neural network architecture $G=(V,E)$, activation $\sigma$, and fixed weights $\vgamma$, we define its {\em  adjacency matrix representation} $\mA(\cdot,\vgamma,\sigma): \vtheta\mapsto\mA(\vtheta,\vgamma,\sigma)$ for each $\vtheta: E_\mathrm{para}\to\sR$,  the image $\mA(\vtheta,\vgamma,\sigma)$ is a matrix whose entries $(\mA(\vtheta,\vgamma,\sigma))_{ij}$ are operators from $\sR$ to $\sR$. More precisely, given any $\vtheta: E_\mathrm{para}\to\sR$, we define
    \begin{equation}
      (\mA(\vtheta,\vgamma,\sigma))_{ij}=\left\{
        \begin{array}{ll}
            \theta_{ij},\quad e_{i\leftarrow j}\in E_\mathrm{para},\\
            \gamma_{ij},\quad e_{i\leftarrow j}\in E_\mathrm{fix},\\
            \sigma(\cdot),\quad e_{i\leftarrow j}\in E_\mathrm{non},\\
            0,\quad \text{otherwise},
        \end{array}
      \right.
    \end{equation}
    where $\theta_{ij}$ and $\gamma_{ij}$ are considered as linear operators from $\sR$ to $\sR$, i.e., multiplication, and $\sigma(\cdot):\sR\to\sR$ refers specifically to the operation of the activation function applied accordingly to its  input.
\end{defi}
}

{ { We remark that Daniely~\cite{Daniely2017SGD} also used DAGs to represent architectures of neural networks. However, the way that the DAGs is used in our paper is different from~\cite{Daniely2017SGD} in which the nodes represent nonlinear operations and edges represent linear operations. In our paper, all the operations, linear or nonlinear, are represented by edges. Thus, all nonlinear activations are explicitly represented by entries of the adjacency matrix, which helps us to study various architectures with skip connections.}}

 For simplicity, $\mA(\vtheta,\vgamma,\sigma)$ is denoted by  $\mA(\vtheta,\sigma)$  or even $\mA$ hereafter with no confusion. We claim that $\mA(\vtheta,\vgamma,\sigma)$ is a nonlinear operator acting on $N$ dimensional vector-valued functions,  { where $N$ is given in \Cref{defi...NN} and pre-determined by network architecture. }We impose the nodes consisting of components from the training sample $\vx$ to be source nodes, and the single node of the output function the sink node.  {Since all
connections are from nodes with smaller index to nodes with bigger index, then
 without loss of generality, $\mA$ can be written into a strictly lower triangular matrix. We
also point out that there is no nonzero entries on the main diagonal of $\mA$, since we only
study feedforward networks without recurrence.} More precisely, if we  set the value of the source nodes to be $h_1(\vx),\cdots, h_d(\vx)$ and the sink node  $h_N(\vx)$,  then $\mA$ is of size $N\times N$, and the output  at $h_i(\vx)$ reads inductively for $i>d$,
\begin{equation}\label{Subsection...Adj...Eq...SourceNodeatith..middle}
    h_i(\vx)=\sum_{j: e_{i\leftarrow j}\in E_\mathrm{para}}\theta_{ij}h_j(\vx)+\sum_{j: e_{i\leftarrow j}\in E_\mathrm{fix}}\gamma_{ij}h_j(\vx)+\sum_{j: e_{i\leftarrow j}\in E_\mathrm{non}}\sigma(h_j(\vx)).
\end{equation}
Specifically,  we define that
 {\begin{defi}[Feedforward neural network function]
 Given any feedforward neural network architecture $G=(V,E)$, activation $\sigma$, and fixed weights $\vgamma$, then for each $\vtheta: E_\mathrm{para}\to\sR$, we define its \emph{(feedforward) neural network function} as a mapping $f(\cdot,\vtheta):\sR^{d}\to \sR^{d'}$, where $d$ is  the input dimension of the network, and $d'$ is the output dimension of the network, such that for any training sample $\vx\in\sR^d$,
 \begin{equation}\label{Subsection...Adj...Eq...SourceNodeatNth..end}
 \begin{aligned}
     f(\vx,\vtheta)&:=h_N(\vx)\\&=\sum_{i: e_{N\leftarrow i}\in E_\mathrm{para}}\theta_{Nj}h_i(\vx)+\sum_{i: e_{N\leftarrow i}\in E_\mathrm{fix}}\gamma_{Ni}h_i(\vx)+\sum_{i: e_{N\leftarrow i}\in E_\mathrm{non}}\sigma(h_i(\vx)).
\end{aligned}
 \end{equation}
\end{defi}
}

Next we define the {\em symbol}~ for  the adjacency matrix representation.

\begin{defi}[Symbol]
     Given  $\mA(\vtheta,\vgamma,\sigma)$, we define the {\em symbol} $\mA(\vtheta,\vgamma,\xi)$ as
    \begin{equation}
      (\mA(\vtheta,\vgamma,\xi))_{ij}=\left\{
        \begin{array}{ll}
            \theta_{ij},\quad e_{i\leftarrow j}\in E_\mathrm{para},\\
            \gamma_{ij},\quad e_{i\leftarrow j}\in E_\mathrm{fix},\\
            \xi,\quad e_{i\leftarrow j}\in E_\mathrm{non},\\
            0,\quad \text{otherwise},
        \end{array}
      \right.
    \end{equation}
    where $\xi$ refers to the operation of direct multiplication of the numeric $\xi$ to its input.
\end{defi}
 { In short, the symbol $\mA(\vtheta,\vgamma,\xi)$ is the DAG representation of a linear neural network with
weights $\xi$ where the original network had nonlinearities.}
 {
We remark that the term symbol in this paper is inspired from the definition of the symbol for pseudo-differential operators~\cite{taylor2012pseudodifferential}. But they are quite different because for the latter $\partial_i$ is replaced by $\xi_i$, not a single variable $\xi$ in all dimensions.}
Similarly, we denote $\mA(\vtheta,\vgamma,\xi)$ by  $\mA(\vtheta,\xi)$ hereafter, and  symbol for ResNet illustrated in \Cref{fig....Example} is  shown in \Cref{fig....Realization}. We observe that since $\xi$ performs exactly like the fixed weights $\vgamma$, hence  $\xi$ transforms the whole nonlinear connection elements belonging to   set   $E_\mathrm{non}$ into new elements of set $E_\mathrm{fix}$.  Therefore, the blue dashed   line in \Cref{fig....Example} shall be replaced correspondingly by a solid line in \Cref{fig....Realization}.

 \begin{figure}
 \centering
 \includegraphics[width=\textwidth]{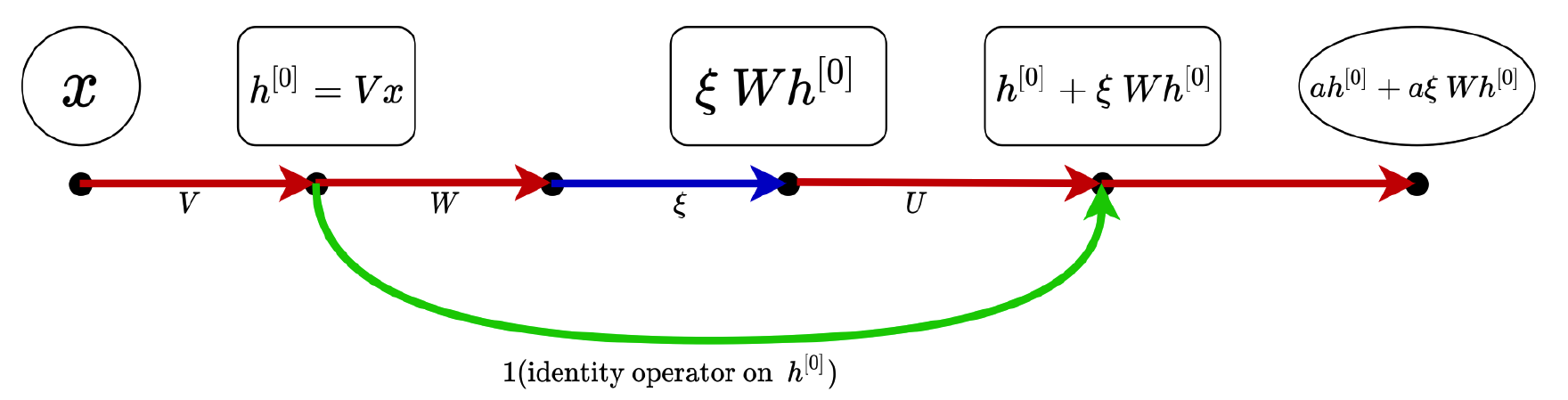}
    \caption{Symbol of the ResNet in \Cref{fig....Example}}
    \label{fig....Realization}
\end{figure}
\subsection{Several Properties of the Adjacency Matrix Representation}\label{subsection...properties}

Given a feedforward neural network, naturally we obtain its adjacency matrix  representation $\mA(\vtheta,\sigma)$,  and for a specific input sample $\vx\in\sR^d$, we define the series of vectors $\left\{\vz_s\right\}_{s=0}^{\infty}$:
\begin{equation}
\begin{aligned}
\vz_0&=(\vx^\T,0,\cdots,0)^\T=\begin{pmatrix}
    \vx\\
    \vzero_{ (N-d)\times 1}
\end{pmatrix}\in\sR^N\\ 
\vz_s&=\vz_0+\mA(\vtheta,\sigma)\vz_{s-1},~~~~s\geq 1.
\end{aligned}
\end{equation}
We observe that  $\mA(\vtheta,\sigma)$ is of size $N\times N$.
Moreover, we define two special vectors in $\sR^N$
\begin{align*}
\vone_\mathrm{in}&=\left(\underbrace{1,1,\cdots, 1}_{\#~\text{of}~1~\text{is}~d},0,0,\cdots, 0\right)^\T=\begin{pmatrix}
    \vone_{d\times 1}\\
    \vzero_{(N-d)\times 1}
\end{pmatrix},\\
\vone_\mathrm{out}&=\left(0,0,\cdots, 0,1 \right)^\T=   \begin{pmatrix}\vzero_{ (N-1)\times 1}\\1\end{pmatrix} ,
\end{align*}
and
the projection matrix $\mP_0$ with respect to the input
$\mP_0=\diag(\vone_\mathrm{in})$,
whose  size is also  $N\times N$.
 We list out several   properties relating to the adjacency matrix representation  $\mA(\vtheta,\sigma)$ and its symbol $\mA(\vtheta,\xi)$.

\begin{prop}\label{prop..property}
 ~\\
    \begin{enumerate}
        \item (Nilpotent)  {For any $\mA(\vtheta,\xi)$,} there exists a positive integer $s_0$, such that for all~$s\geq s_0$, $$\mA^s(\vtheta,\xi)=\mzero.$$ \\
        \item  {($L^\infty$ bound of the finite difference)} For all $s\geq 1$, $$\vz_{s+1}-\vz_s=\mA(\vtheta,\sigma)\vz_s-\mA(\vtheta,\sigma)\vz_{s-1},$$ and $$\norm{\vz_{s+1}-\vz_s}_\infty\leq \norm{\mA^s(\abs{\vtheta},1)\Abs{\vz_1-\vz_0}}_\infty,$$
        where $\Abs{\vz }$ with $\vz$ being a vector means taking the absolute values of all the entries of the vector.
        \item (Limit) There exists a limit for  the series of vectors $\left\{\vz_s\right\}_{s=1}^{\infty}$, i.e., $$\vz_\infty=\lim_{s\to\infty}\vz_s.$$\\
        \item (Representation for network output) The output function of the neural network $f(\vx)$  reads  $$f(\vx)=\vone_\mathrm{out}^\T\vz_\infty.$$\\
        \item (Fixed point iteration) Define $\bar{\mA}=\mP_0+\mA$, then $$\vz_s=\bar{\mA}\vz_{s-1}.$$  
         {Thus, $\vz_\infty=\bar{\mA}^\infty\vz_0:=\lim_{s\to+\infty}\bar{\mA}^s\vz_0$ exists and is a fixed point for the operator $\bar{\mA}(\vtheta,\sigma)$.}
        \item (Alternate expression of weighted path norm) Given $\vtheta$,
        its weighed path norm reads \begin{equation}\label{Prop...eq..AlternateWeightedPathNorm}
        \begin{aligned}
 \norm{\vtheta}_\mathrm{P}&=\vone_\mathrm{out}^\T\sum_{s=0}^\infty\mA^s(\abs{\vtheta},3)\vone_\mathrm{in} =\vone_\mathrm{out}^{\T}(\mI_{N\times N}-\mA(\abs{\vtheta},3))^{-1}\vone_\mathrm{in}\\&=\vone_\mathrm{out}^\T\bar{\mA}^\infty(\abs{\vtheta},3)\vone_\mathrm{in}.
        \end{aligned}
        \end{equation}

          \item (Number of parameters and nonlinear connections) Given symbol $\mA(\vtheta,\vgamma,\xi)$, then
        \begin{align*}
          N_{\mathrm{para}}=\#E_\mathrm{para}
          &= \norm{\mA(\vone_{\vtheta},\vzero,0)}_{1,1},\\
          N_{\mathrm{fix}}=\#E_\mathrm{fix}
          &= \norm{\mA(\vzero,\vone_{\vgamma},0)}_{1,1},\\
            N_{\mathrm{non}}=\#E_\mathrm{non}
          &= \norm{\mA(\vzero,\vzero,1)}_{1,1},
        \end{align*}
        where $\vone_{\vtheta}$ is obtained by replacing all the components of $\vtheta$ by $1$, and $\vone_{\vgamma}$ is attained similarly by replacing all the components of $\vgamma$ by $1$.
    \end{enumerate}
\end{prop}

\begin{rmk}
   {The proof of \Cref{prop..property} is given in \Cref{appendix},} and the alternate expression of weighted path norm  in \eqref{Prop...eq..AlternateWeightedPathNorm} is  useful in the proof of \Cref{lm:property_H} and \Cref{thm..ApproxError}.
\end{rmk}

\subsection{Examples}\label{subsection....examples}
Individually, the  adjacency matrix representation for each  feedforward neural network mentioned beforehead in \Cref{subsection....FeedForward} is presented as follows.

\begin{exam}[Two-layer network]
    The adjacency matrix representation $\mA$ of the two-layer network \eqref{eq..Architecture2Layer} reads 
  \begin{equation*}
    \begin{pmatrix}
      \mzero_{d\times d}&&&\\
      \mW&\mzero_{m\times m}&&\\
      &\sigma \mI_{m\times m}&\mzero_{m\times m}&\\
      &&\vu^\T&0
    \end{pmatrix}.
  \end{equation*}
\end{exam}

 {
\begin{exam}[Fully connected network] \label{ex...Fully}
    For fully connected deep network \eqref{eq..ArchitectureFCNet}, the adjacency matrix representation $\mA$ takes the matrix form   
    \begin{equation*}
        \begin{pmatrix}
            \mzero & & & & & & &\\
             \mB^{[1]} & \mzero & & & & &\\
             & \mB^{[2]} & \mzero & & & &\\
             & & \ddots & \ddots & & &\\
             & & & \mB^{[l]} & \mzero & &\\
             & & & & \ddots & \ddots &\\
             & & & & & \mB^{[L]} &\mzero &\\
             & & & & & & \vu^\T &\mzero
        \end{pmatrix},
    \end{equation*}
where for each  $l=1,\cdots,L$,  the matrix block $\mB^{[l]}$ reads
  \begin{equation*}
    \mB^{[l]}=\begin{pmatrix}
     \mW^{[l]}&\mzero_{m_{l}\times m_{l}}\\
     &\sigma\mI_{m_{l}\times m_{l}}
    \end{pmatrix}.
  \end{equation*}

\end{exam}
}
 {
For \Cref{ex...Resnet} and \Cref{ex...Densenet}, the matrix representation $\mA$ incorporates the form
\begin{equation*}
        \begin{pmatrix}
            \mzero & & & & & & &\\
            \mV & \mzero & & & & & &\\
            & \mB^{[1]} & \mzero & & & & &\\
            & & \mB^{[2]} & \mzero & & & &\\
            & & & \ddots & \ddots & & &\\
            & & & & \mB^{[l]} & \mzero & &\\
            & & & & & \ddots & \ddots &\\
            & & & & & & \mB^{[L]} &\mzero &\\
            & & & & & & & \vu^\T &\mzero
        \end{pmatrix},
    \end{equation*}
    with  $\mB^{[l]}$ to be specified for each case. }
     {
\begin{exam}[ResNet]\label{ex...Resnet}
  For  ResNet \eqref{eq..ArchitectureResNet}, each $\mB^{[l]}$ reads 
  \begin{equation*}
    \mB^{[l]}=\begin{pmatrix}
      \mW^{[l]}&\mzero_{m\times m}&\\
      &\sigma\mI_{m\times m} &\mzero_{m\times m}\\
      \mI_{D\times D}&&\mU^{[l]}
    \end{pmatrix}.
  \end{equation*}
\end{exam}
}
 {
\begin{exam}[DenseNet]\label{ex...Densenet}
   For DenseNet \eqref{eq..ArchitectureDenseNet}, each $\mB^{[l]}$ reads
  \begin{equation*}
    \mB^{[l]}=
    \begin{pmatrix}
      \mW^{[l]}  & \mzero_{lm\times lm} &  \\
        &\sigma\mI_{lm\times lm}  &\mzero_{lm\times lm} \\
      \bar{\mI}_{(k_0+lk)\times (k_0+(l-1)k)}  & &\bar{\mU}^{[l]}  
    \end{pmatrix},
  \end{equation*}
  where for each $l=1,\cdots, L$,
  \begin{align*}
    \bar{\mI}_{(k_0+lk)\times (k_0+(l-1)k)}=
    \begin{pmatrix}
      \mI_{(k_0+(l-1)k)\times(k_0+(l-1)k)}\\\mzero_{k\times (k_0+(l-1)k)}
    \end{pmatrix},\quad
    \bar{\mU}^{[l]}=
    \begin{pmatrix}
      \mzero_{(k_0+(l-1)k)\times lm}\\\mU^{[l]}
    \end{pmatrix}.
  \end{align*}
\end{exam}
}

\section{Main Results}\label{Section...Main Results}

\subsection{Setup}
The goal of the supervised learning is to find a network function that fits the training samples and also generalizes well on  test data. Our problem of interest is to learn a function from a   sample dataset of $n$ examples $S:=\{(\vx_i,y_i)\}_{i=1}^n$ drawn i.i.d  from an underlying distribution $\fD$, where for each $i$, $\vx_i\in\Omega=[0,1]^d$, and our target function is $f^{*}:\Omega \to [0,1]$ with $y_i=f^{*}(\vx_i)\in[0,1]$.
Similar to the cases of  ResNet~\cite{e2019prioriRes} and two-layer~\cite{ew2019prioriTwo}, a  truncation operator shall be defined such that for any function $h:\sR^d\to \sR$, $\fT_{[0,1]}h(\vx)=\min\left\{\max\{h(\vx),0\},1\right\}$.
With an abuse of notation, we still use $f$ to denote $\fT_{[0,1]}f$ henceforth. Consider  the truncated square
loss \begin{equation}\label{Subsection....mainResults...Eq...TruncationPopulationRisk}
    \ell(\vx,\vtheta)=\frac{1}{2}\Abs{\fT_{[0,1]}f(\vx;\vtheta)-f(\vx)}^2
\end{equation} 
in the sequel,  then the  empirical risk is defined as 
\begin{equation}\label{eq..EmpiricalRisk}
   \RS(\vtheta)=\frac{1}{n}\sum_{i=1}^{n}\ell(\vx_i,\vtheta),
\end{equation} 
and the  population risk is defined as 
\begin{equation}\label{Subsection...MainREsults....eq..PopulationRisk}
  \RD(\vtheta)=\Exp_{\vx\sim\fD}\ell(\vx,\vtheta).
\end{equation}
The  ultimate goal of our paper is to minimize $\RD(\vtheta)$.
\subsection{Main Theorems}
We consider a feedforward neural network with its adjacency matrix representation satisfying the following assumptions.

\begin{assump}\label{assump..InputLinPara}
   { Given any feedforward neural network architecture $G=(V,E)$, activation $\sigma$, and fixed weights $\vgamma$}, we assume that  for any edge~$e_{i\leftarrow j}$~with $j\leq d$,  $e_{i\leftarrow j}\in E_{\mathrm{para}}$. In other words, we assume that there exists no edge~$e_{i\leftarrow j}$ with $j\leq d$, such that $e_{i\leftarrow j}\in E_{\mathrm{fix}}$ or $e_{i\leftarrow j}\in E_\mathrm{non}$.
\end{assump}
\begin{assump}[Shortcut connections]\label{assump..BlockMatrix}
 { Given any feedforward neural network architecture $G=(V,E)$, activation $\sigma$, and fixed weights $\vgamma$,} we assume that for any $\vtheta:E_\mathrm{para}\to\sR$ its adjacency matrix representation takes either the form
      \begin{equation}\label{Definition...Eq...TwoLayerRepresentation}
        \mA(\vtheta,\vgamma,\sigma)=\begin{pmatrix}
          \mzero&&&\\
          \vV&\mzero&&\\
          &\sigma \mI&\mzero&\\
          &&\vu^\T&0
        \end{pmatrix},
      \end{equation}
or the form
    \begin{equation}\label{Definition...Eq...LongRepresentation}
        \mA(\vtheta,\vgamma,\sigma)=
        \begin{pmatrix}
            \mzero & & & & & & &\\
            \mV & \mzero & & & & & &\\
            & \mB^{[1]} & \mzero & & & & &\\
            & & \mB^{[2]} & \mzero & & & &\\
            & & & \ddots & \ddots & & &\\
            & & & & \mB^{[l]} & \mzero & &\\
            & & & & & \ddots & \ddots &\\
            & & & & & & \mB^{[L]} &\mzero &\\
            & & & & & & & \vu^\T &\mzero
        \end{pmatrix},
    \end{equation}
    with    $\mV$ taking size  $d_0\times d$, length of  vector $\vu$ being  $d_L$, and for each matrix block~$\mB^{[l]}$:
    \begin{equation*}
        \mB^{[l]}=
        \begin{pmatrix}
            \mW^{[l]}_{} &  &\\
             & \sigma\mI_{p_l\times p_l} & \\
            \mS^{[l]} &  & \mU^{[l]}
        \end{pmatrix},
    \end{equation*}
    where  $\mB^{[l]}$ has size  $(2p_{l}+d_{l})\times(2p_{l}+d_{l-1})$,   its components $\left\{\mW^{[l]},\sigma\mI_{p_l\times p_l}, \mS^{[l]},  \mU^{[l]}\right\}$ respectively  has size  $p_l\times d_{l-1}$, $p_l\times p_l$, $d_{l}\times d_{l-1}$, and $d_{l}\times p_l$. Moreover, we assume further that $\mS^{[l]}$ is   a  row permutation matrix of 
    $\begin{pmatrix}
      \mI_{d_{l-1}\times d_{l-1}}\\\mzero_{(d_{l}-d_{l-1})\times d_{l-1}}
    \end{pmatrix}$, and for all $l=1,\cdots,L$, it holds that  {$d_l \geq d_{l-1}, \min\{d_0, d_l, p_l\} \geq d+1$.} 
\end{assump}
Evidently, any feedforward neural network satisfying \Cref{assump..BlockMatrix}   satisfies \Cref{assump..InputLinPara}. With this in mind, we proceed to the statement of our main theorems. 

\begin{thm}[Approximation error]\label{thm..ApproxError}
    { Given any feedforward neural network architecture $G=(V,E)$, activation $\sigma$, fixed weights $\vgamma$, suppose that \Cref{assump..BlockMatrix} holds,} then for any target function $f^*\in\fB$, there exists a feedforward neural network function $f(\cdot;\tilde{\vtheta})$   with $\norm{\tilde{\vtheta}}_\mathrm{P}\leq 6\norm{f^*}_\fB$, such that
  \begin{equation}\label{Thm...eq..ApproErrorEstimate}
    \RD(\tilde{\vtheta}):=\Exp_{x\sim\fD}\tfrac{1}{2}(f(\vx;\tilde{\vtheta})-f^*(\vx))^2\leq \frac{3\norm{f^*}^2_\fB}{2N_{\mathrm{non}}}.
  \end{equation}
\end{thm}

\begin{thm}[A posteriori estimate]\label{thm..Aposteriori}
 {  Given any feedforward neural network architecture $G=(V,E)$, activation $\sigma$, fixed weights $\vgamma$,  suppose that  \Cref{assump..BlockMatrix} holds,} then for any $\delta\in(0,1)$, with probability at least $1-\delta$ over the choice of the  training sample $S$, the following holds
    \begin{equation}
        \abs{\RD(\vtheta) - \RS(\vtheta)}
        \leq(\norm{\vtheta}_\mathrm{P}+1)\frac{6\sqrt{2
        \log(2d)}+\frac{1}{2\sqrt{2}}}{\sqrt{n}} + \frac{1}{2}\sqrt{\frac{
        \log(\frac{\pi^2}{3\delta})}{2n}}.
    \end{equation}
\end{thm}

 {
\begin{thm}[A priori estimate]\label{thm..Apriori}
  Given any feedforward neural network architecture $G=(V,E)$, activation $\sigma$, fixed weights $\vgamma$, suppose that \Cref{assump..BlockMatrix} holds, $f^*\in\fB$, $\lambda= \Omega(\sqrt{\log d})$, and that  $\vtheta_{S,\lambda}$ is an optimal solution for the regularized model \begin{equation}\label{Subsection...MainREsults...eq..RegularizedRisk}
    J_{S,\lambda}(\vtheta):=\RS(\vtheta)+\frac{\lambda}{\sqrt{n}}\norm{\vtheta}_\mathrm{P},
\end{equation} 

  then  for any $\delta\in(0,1)$, with probability at least $1-\delta$ over the choice of the training sample $S$,  the population risk satisfies
  \begin{equation}\label{eq..AprioriEstimate}
    \begin{aligned}
        \RD(\vtheta_{S,\lambda})
        &:=\Exp_{\vx\sim\fD}\tfrac{1}{2}(f(\vx;\vtheta_{S,\lambda})-f^*(\vx))^2\\
        &\lesssim \frac{\norm{f^*}_{\fB}^2}{N_\mathrm{non}}+  \frac{1}{\sqrt{n}}
        \left(\lambda(\norm{f^*}_{\fB}+1)+\sqrt{\log 1/\delta}\right).
    \end{aligned}
  \end{equation}
\end{thm}
}
\begin{rmk}\label{rmk...}
  As is shown in \Cref{subsection....examples},
  some of the  examples mentioned in \Cref{subsection....FeedForward}, i.e., Two-layer Network,  ResNet and DenseNet satisfy \Cref{assump..BlockMatrix}.   
\end{rmk}

\section{Proof of Theorems and Applications}\label{Section...Proofs}

\subsection{Approximation Error}
First and foremost, in order to prove \Cref{thm..ApproxError}, we  recall the  result  obtained in~\cite[Proposition 2.1]{ew2019prioriTwo}.
\begin{thm}[Approximation error for two-layer networks]\label{thm..ApproErrorTwoLayer}
  For any target function $f^*\in\fB$, there exists a two-layer network $f_\mathrm{2Layer}(\cdot;\vtheta_\mathrm{2Layer})$ of width $m$, such that
  \begin{equation}\label{eq..Approximation2LayerIneq1}
   \Exp_{\vx\sim\fD} \tfrac{1}{2}\left(f_\mathrm{2Layer}(\vx;\vtheta_\mathrm{2Layer})-f^*(\vx)
    \right)^2
    \leq \frac{3\norm{f^*}^2_\fB}{2m},
    \end{equation}
   with  its parameters $\vtheta_\mathrm{2Layer}=\left\{a_k,\vw_k\right\}_{k=1}^m$ satisfying  \begin{equation}\label{eq..Approximation2LayerIneq2}
       \sum_{k=1}^m\abs{a_k}\norm{\vw_k}_1
    \leq 2\norm{f^*}_\fB, 
   \end{equation}and the output reads $f_\mathrm{2Layer}(\vx;\vtheta_\mathrm{2Layer})=\sum_{k=1}^m a_k\sigma(\vw^\T_k\vx)$.
  \end{thm}
\begin{proof}[Proof of   \Cref{thm..ApproxError}]
Given a feedforward network $f(\cdot;\tilde{\vtheta})$ with input dimension $d$ and its adjacency matrix representation:
  \begin{equation*} 
        \mA(\vtheta,\vgamma,\sigma)=\begin{pmatrix}
          \mzero&&&\\
          \vV&\mzero&&\\
          &\sigma \mI&\mzero&\\
          &&\vu^\T&0
        \end{pmatrix}.
      \end{equation*}
      
Representation belonging to this case can be easily proved by observing that the total number of nonlinearies for a two-layer network with width $m$ is exactly $m$, then we may apply \Cref{thm..ApproErrorTwoLayer} directly to obtain the results.

In the case of its adjacency matrix representation being:
    \begin{equation*} 
        \mA(\vtheta,\vgamma,\sigma)=
        \begin{pmatrix}
            \mzero & & & & & & &\\
            \mV & \mzero & & & & & &\\
            & \mB^{[1]} & \mzero & & & & &\\
            & & \mB^{[2]} & \mzero & & & &\\
            & & & \ddots & \ddots & & &\\
            & & & & \mB^{[l]} & \mzero & &\\
            & & & & & \ddots & \ddots &\\
            & & & & & & \mB^{[L]} &\mzero &\\
            & & & & & & & \vu^\T &\mzero
        \end{pmatrix}.
    \end{equation*}
Without loss of generality, for all $l=1,2,\cdots, L$, set   $\mS^{[l]}$ as
\begin{equation*}
    \mS^{[l]}=\begin{pmatrix}
      \mI_{d_{l-1}\times d_{l-1}}\\\mzero_{(d_{l}-d_{l-1})\times d_{l-1}}
    \end{pmatrix}. 
\end{equation*}

From \Cref{thm..ApproErrorTwoLayer},  there exists a two-layer network $f_\mathrm{2Layer}(\cdot;\vtheta_\mathrm{2Layer})$ of width $m_L$, with its parameters $\vtheta_\mathrm{2Layer}=\left\{a_k,\vw_k\right\}_{k=1}^{m_L}$ satisfying $\sum_{k=1}^{m_L}\abs{a_k}\norm{\vw_k}_1\leq 2\norm{f^*}_\fB$,
and the output reads $f_\mathrm{2Layer}(\vx;\vtheta_\mathrm{2Layer})=\sum_{k=1}^{m_L} a_k\sigma(\vw^\T_k\vx)$, fulfilling that

  \begin{equation*}
   \Exp_{\vx\sim\fD} \tfrac{1}{2}\left(f_\mathrm{2Layer}(\vx;\vtheta_\mathrm{2Layer})-f^*(\vx)
    \right)^2
    \leq \frac{3\norm{f^*}^2_\fB}{2m_L}.
    \end{equation*}
Set $m_l=\sum_{k=1}^l p_k$, and  $m_0=0$ for the purpose of completion, we  notice that $N_{non}=m_L=\sum_{k=1}^L p_k$. 

Existence of the feedforward network  $f(\cdot;\tilde{\vtheta})$ shall be proved by construction. For  each $l=1,2,\cdots, L$, we have
\begin{align*}
    \mV&=\begin{pmatrix}
      \mI_{d\times d} \\ \mzero_{(d_0-d)\times d}
    \end{pmatrix},\\\vu&=\begin{pmatrix}
      \vzero_{(d_L-1)\times 1}\\ 1
    \end{pmatrix},\\
    \mW^{[l]}&=\begin{pmatrix}
      \vw^\T_{(m_{l-1})+1}&\vzero_{1\times (d_{l-1}-d)}\\
         \vw^\T_{(m_{l-1})+2}&\vzero_{1\times (d_{l-1}-d)}\\
         \vdots&\vdots\\
         \vw^\T_{(m_{l-1})+p_{l}}&\vzero_{1\times (d_{l-1}-d)}
    \end{pmatrix},\\
    \mU^{[l]}&=\begin{pmatrix}
      \vzero_{  (d_{l}-1)\times 1} & \vzero_{  (d_{l}-1)\times 1}&\cdots&\vzero_{  (d_{l}-1)\times 1}\\
            a_{(m_{l-1})+1}&a_{(m_{l-1})+2}&\cdots&a_{(m_{l-1})+p_{l}}
    \end{pmatrix}.
\end{align*}
One can easily verify that $\tilde{\vtheta}=\mathrm{vec}\left\{\mV,\{\mW^{[l]},\mU^{[l]}\}_{l=1}^L,\vu\right\}$, with
\begin{equation*}
    \norm{\tilde{\vtheta}}_\mathrm{P}=\vone_\mathrm{out}^\T\bar{\mA}^\infty\left(\abs{\tilde{\vtheta}},3\right)\vone_\mathrm{in}=3\sum_{j=1}^{m_L}\Abs{a_j}\Norm{\vw_j}_1 \leq 6\norm{f^*}_\fB.
  \end{equation*}
  Moreover,  {$f(\vx;\tilde{\vtheta})=f_\mathrm{2Layer}(\vx;\vtheta_\mathrm{2Layer})=\sum_{j=1}^{m_L}a_j\sigma\left(\vw_j^\T\vx \right)$,} thus \begin{equation*}
   \Exp_{\vx\sim\fD} \tfrac{1}{2}\left(f(\vx;\tilde{\vtheta})-f^*(\vx)
    \right)^2
    \leq \frac{3\norm{f^*}^2_\fB}{2m_L},\end{equation*} 
  since  the total number of nonlinearies $N_{non}=\sum_{k=1}^L p_k=m_L$, we  finish our proof.
    
\end{proof}

\subsection{Rademacher Complexity}
In this part, we endeavor to   bound the (empirical) Rademacher complexity of  networks with path norm $\Norm{\vtheta}_{\mathrm{P}}\leq Q$. Let $\fH^{N}_Q=\{f(\cdot;\vtheta): \norm{\vtheta}_{\mathrm{P}}\leq Q\}$ be the set of feedforward neural networks satisfying \Cref{assump..InputLinPara} with a total of $N$ nodes, $N>d$.  {It is evident that 
  for any fixed $N$ and  $Q>0$, $0\in\fH^{N}_Q$, where $0$ refers to  the zero  function that maps any input to the numeric $0$, i.e., for all $\vx\in\sR^d,~0(\vx)\equiv 0$.}
 
\begin{lem}\label{lm:property_H}
     For any fixed $N$ and  $Q>0$, $\fH^{N}_Q\subseteq\fH^{\overline{N}}_Q$, for all $\overline{N}>N$. Moreover, $\fH^{N}_Q=Q\fH^{N}_1$.
\end{lem}

\begin{proof}
    From the  {positive homogeneity} of ReLU, it is obvious that $\fH^{N}_Q=Q\fH^{N}_1$. 
    
    We proceed to prove $\fH^{N}_Q\subseteq\fH^{\overline{N}}_Q$. For any $f(\cdot;\vtheta)\in\fH^{N}_Q$, then $\norm{\vtheta}_{\mathrm{P}}\leq Q$. Let $\mA$ be its adjacency matrix representation, then $\mA$ is of size $N\times N$. Let $\mE$ be a matrix of size $\left(\overline{N}-N\right)\times N$ with its bottom right entry being $1$, and other components equal to zero, i.e.,
    $$ \mE=\begin{pmatrix}
      \vzero_{(\bar{N}-N-1)\times (N-1)}&   \vzero_{(\bar{N}-N-1)\times 1}\\
            \vzero_{1 \times(N-1)}&1
    \end{pmatrix},$$
  Moreover, we  set
    \begin{equation*}
       \widetilde{\mA}=\begin{pmatrix} 
            \mA & \mzero_{N\times (\overline{N}-N)} \\
            \mE & \mzero_{(\overline{N}-N)\times (\overline{N}-N)} 
        \end{pmatrix}.
    \end{equation*} 
    Then, $\widetilde{\mA}$ is of size ${\overline{N}\times\overline{N}}$, and it is the adjacency matrix representation of a feedforward neural network $\widetilde{f}(\cdot;\overline{\vtheta})$ with $\overline{N}$ nodes. Thus, for some $\overline{Q}>0$,
    $
       \widetilde{f}(\cdot;\overline{\vtheta})\in\fH^{\overline{N}}_{\overline{Q}}
    $.
    
    Next, we  need to compute the path norm of $\widetilde{f}(\cdot;\overline{\vtheta})$. 
Let $$\vone_\mathrm{in}^N=\begin{pmatrix}
    \vone_{d\times 1}\\
    \vzero_{(N-d)\times 1}
\end{pmatrix}, \vone^{\overline{N}}_\mathrm{in}=\begin{pmatrix}
    \vone_{d\times 1}\\
    \vzero_{(\overline{N}-d)\times 1}
\end{pmatrix},
\vone_\mathrm{out}^N=\begin{pmatrix}
    \vzero_{(N-1)\times 1}\\ 1
\end{pmatrix}, 
\vone^{\overline{N}}_\mathrm{in}=\begin{pmatrix}
    \vzero_{(\overline{N}-1)\times 1}\\ 1
\end{pmatrix}, $$
then directly from \eqref{Prop...eq..AlternateWeightedPathNorm} obtained in \Cref{prop..property}, we have
    \begin{align*}
    \Norm{\overline{\vtheta}}_{\mathrm{P}}&= \sum\limits_{k=1}^\infty (\vone^{\overline{N}}_\mathrm{out})^\T\begin{pmatrix} 
        \mA^k & \mzero_{N\times(\overline{N}-N)} \\
        \mE\mA^{k-1} & \mzero_{(\overline{N}-N)\times(\overline{N}-N)}
    \end{pmatrix}
    \vone^{\overline{N}}_\mathrm{in} \\
      &= (\vone^{\overline{N}-N}_\mathrm{out})^\T \left(\sum\limits_{k=1}^\infty \mE\mA^k\right) \vone^N_\mathrm{in}= (\vone^N_\mathrm{out})^\T\left(\sum\limits_{k=1}^\infty  \mA^k\right)\vone^N_\mathrm{in}=\Norm{ {\vtheta}}_{\mathrm{P}}\leq Q.
    \end{align*}
    Hence, it holds that  $ \widetilde{f}(\cdot;\overline{\vtheta})\in\fH^{\overline{N}}_Q$.
    
    Finally, since the function outputs satisfy that for all $\vx\in\sR^d$,  $\widetilde{f}(\vx;\overline{\vtheta})=f(\vx;\vtheta)$, we conclude that   $\fH^{N}_Q\subseteq\fH^{\overline{N}}_Q$.
\end{proof}
 We set $\fH^{N}=\bigcup_{Q>0}\fH^{N}_Q$, then
the next lemma gives a decomposition for any network in $\fH^{N}$.
\begin{lem}\label{lemma...Decomposition}
Given an  input sample~$\vx=\left(x_1,x_2,\cdots, x_d\right)^\T$, for any $f^N\in\fH^N$, it can be decomposed into linear and nonlinear parts$\mathrm{:}$
\begin{equation}\label{Lemma..Decomposiotion..eq...DecomponAlgebra}
    f^{ {N}}(\vx) = \sum\limits_{1\leq i\leq d} a_i x_i + \sum\limits_{d+1\leq i \leq N-1} a_i \sigma(f^i(\vx)),
\end{equation}
where for each index $i$, $a_i$ is a  scalar, and $f^i\in\fH^{i}$.

Moreover, we have
\begin{equation}\label{Lemma..Decomposiotion..ineq...Decompnorm}
    \left( \sum\limits_{1\leq i\leq d} \abs{a_i}+ 3\sum\limits_{d+1\leq i \leq N-1}\abs{a_i}\norm{f^i}_\mathrm{P}\right)\leq \norm{f^{ {N}}}_\mathrm{P}.
\end{equation}
\end{lem}
 Lemma \ref{lemma...Decomposition}   are essentially proved using mathematical induction.   \eqref{Lemma..Decomposiotion..eq...DecomponAlgebra} reveals the components of  function $f^N(\vx)$, and \eqref{Lemma..Decomposiotion..ineq...Decompnorm} is essentially proved  by repeatedly using triangle inequality.
\begin{proof}
We prove \eqref{Lemma..Decomposiotion..eq...DecomponAlgebra} and \eqref{Lemma..Decomposiotion..ineq...Decompnorm}  by induction.  Firstly, when ${N}=d+1$,  then directly from \Cref{assump..InputLinPara}, we have for any $f^{d+1}\in\fH^{d+1}$, there exists coefficients $a_i$ with $1\leq i\leq d$, such that 
\begin{equation*}
    f^{d+1}(\vx) = \sum\limits_{1\leq i\leq d}  a_i x_i,
\end{equation*}
then \eqref{Lemma..Decomposiotion..eq...DecomponAlgebra} holds trivially. Moreover,
\begin{equation*}
    \norm{f^{d+1}}_\mathrm{P}=  \sum\limits_{1\leq i\leq d} \abs{a_i},
\end{equation*}
then  inequality~\eqref{Lemma..Decomposiotion..ineq...Decompnorm} on path norm also holds trivially.

Secondly, we assume that  \eqref{Lemma..Decomposiotion..eq...DecomponAlgebra} and \eqref{Lemma..Decomposiotion..ineq...Decompnorm} holds for  $ d+1, d+2, \cdots, N$, then we proceed to show that they hold true for $N+1$. For any $f^{N+1}\in\fH^{N+1}$, we have 
 {
\begin{equation}\label{eq....induction.....N+1}
\begin{aligned}
    &f^{N+1}(\vx) = \sum\limits_{1\leq i\leq d,~e_{N+1\leftarrow i}\in E_{\mathrm{para}}}w(e_{N+1\leftarrow i})x_i\\&+\sum\limits_{d+1\leq i\leq N,~e_{N+1\leftarrow i}\in E_{\mathrm{para}}\sqcup E_{\mathrm{fix}}}w(e_{N+1\leftarrow i})f^{i}(\vx)+\sum\limits_{d+1\leq i\leq N,~e_{N+1\leftarrow i}\in E_\mathrm{non}}\sigma(f^{i}(\vx)).
    \end{aligned}
\end{equation}}
From the induction hypothesis, for any index $i$ with $d
+1\leq i\leq N$, then there exist constants $a_{i,j}$ with $d
+1\leq j\leq i-1$, such that
\begin{equation}\label{eq....induction....i}
    f^{i}(\vx) = \sum\limits_{1\leq j\leq d} a_{i,j} x_j + \sum\limits_{d+1\leq j\leq i-1}a_{i,j}\sigma(f^j(\vx)).
\end{equation}
By plugging~\eqref{eq....induction....i} into~\eqref{eq....induction.....N+1},   we obtain that
 {
\begin{align*}
     f^{N+1}(\vx) &= \sum\limits_{1\leq i\leq d,~e_{N+1\leftarrow i}\in E_{\mathrm{para}}}w(e_{N+1\leftarrow i})x_i\\
     +\sum\limits_{d+1\leq i\leq N,~e_{N+1\leftarrow i}\in E_{\mathrm{para}}\sqcup E_{\mathrm{fix}}}&w(e_{N+1\leftarrow i})\left(\sum\limits_{1\leq j\leq d} a_{i,j} x_j + \sum\limits_{d+1\leq j\leq i-1}a_{i,j}\sigma(f^j(\vx))\right)\\
     +\sum\limits_{d+1\leq i\leq N,~e_{N+1\leftarrow i}\in E_\mathrm{non}}&\sigma(f^{i}(\vx))=\sum\limits_{1\leq i\leq d}a_ix_i+\sum\limits_{d+1\leq i \leq N} a_i \sigma(f^i(\vx)),
\end{align*}
}
where for $1\leq i\leq d$,  
    \begin{equation}\label{eq..decop..part1}
      a_i=\left\{
        \begin{array}{ll}
           w(e_{N+1\leftarrow i})+ \sum\limits_{d+1\leq j\leq N,~e_{N+1\leftarrow j}\in E_{\mathrm{para}}\sqcup E_{\mathrm{fix}}}  w(e_{N+1\leftarrow j})a_{j,i},\quad  e_{N+1\leftarrow i}\in E_{\mathrm{para}},\\
           \sum\limits_{d+1\leq j\leq N,~e_{N+1\leftarrow j}\in E_{\mathrm{para}}\sqcup E_{\mathrm{fix}}}  w(e_{N+1\leftarrow j})a_{j,i},\quad \text{otherwise},
        \end{array}
      \right.
    \end{equation}
    and for $d+1\leq i\leq N$, 
    \begin{equation}\label{eq..decop..part2}
      a_i=\left\{
        \begin{array}{ll}
          1+ \sum\limits_{i+1\leq j\leq N,~e_{N+1\leftarrow j}\in E_{\mathrm{para}}\sqcup E_{\mathrm{fix}}}  w(e_{N+1\leftarrow j})a_{j,i},\quad e_{N+1\leftarrow i}\in E_\mathrm{non},\\
           \sum\limits_{i+1\leq j\leq N,~e_{N+1\leftarrow j}\in E_{\mathrm{para}}\sqcup E_{\mathrm{fix}}}  w(e_{N+1\leftarrow j})a_{j,i},\quad \text{otherwise},
        \end{array}
      \right.
    \end{equation}
\eqref{eq..decop..part1} and \eqref{eq..decop..part2} guarantee 
existence of the  decomposition~\eqref{Lemma..Decomposiotion..eq...DecomponAlgebra} for the case $N+1$.

Finally, for the norm inequality~\eqref{Lemma..Decomposiotion..ineq...Decompnorm}, we  notice that from~\eqref{eq....induction.....N+1},
\begin{equation}\label{proofLemma...eq..tobecontinued}
\begin{aligned}
    \norm{f^{N+1}}_\mathrm{P} 
    &=  \sum\limits_{1\leq i\leq d,~e_{N+1\leftarrow i}\in E_{\mathrm{para}}} \abs{w(e_{N+1\leftarrow i}) }\\
    &~~+\sum\limits_{d+1\leq i\leq N,~e_{N+1\leftarrow i}\in E_{\mathrm{para}}\sqcup E_{\mathrm{fix}}}\abs{w(e_{N+1\leftarrow i})}\norm{f^{i}}_\mathrm{P} \\
    &~~+ 3\sum\limits_{d+1\leq i\leq N,~e_{N+1\leftarrow i}\in E_\mathrm{non}}\norm{f^{i}}_\mathrm{P}\\
    &\geq \sum\limits_{1\leq i\leq d,~e_{N+1\leftarrow i}\in E_{\mathrm{para}}}\abs{w(e_{N+1\leftarrow i}) }\\
    +\sum\limits_{d+1\leq i\leq N,~e_{N+1\leftarrow i}\in E_{\mathrm{para}}\sqcup E_{\mathrm{fix}}} &\abs{w(e_{N+1\leftarrow i})}\left(\sum\limits_{1\leq j\leq d} \abs{a_{i,j}} + 3\sum\limits_{d+1\leq j\leq i-1}\abs{a_{i,j}}\norm{f^j}_{\mathrm{P}}\right)\\
    &~~+  3\sum\limits_{d+1\leq i\leq N,~e_{N+1\leftarrow i}\in E_\mathrm{non}}\norm{f^{i}}_\mathrm{P}\\&= \sum\limits_{1\leq i\leq d}b_i+\sum\limits_{d+1\leq i \leq N} b_i\norm{f^{i}}_\mathrm{P},
\end{aligned}
\end{equation}
where for $1\leq i\leq d$,  
    \begin{equation}\label{eq..decop..part1..rep}
      b_i=\left\{
        \begin{array}{ll}
           \abs{w(e_{N+1\leftarrow i})}+\sum\limits_{d+1\leq j\leq N,~e_{N+1\leftarrow j}\in E_{\mathrm{para}}\sqcup E_{\mathrm{fix}}}  \abs{w(e_{N+1\leftarrow j})}\abs{a_{j,i}}, e_{N+1\leftarrow i}\in E_{\mathrm{para}},\\
          \sum\limits_{d+1\leq j\leq N,~e_{N+1\leftarrow j}\in E_{\mathrm{para}}\sqcup E_{\mathrm{fix}}}  \abs{w(e_{N+1\leftarrow j})}\abs{a_{j,i}},\quad \text{otherwise},
        \end{array}
      \right.
    \end{equation}
    then inequality $b_i\geq \abs{a_i}$ holds for all $1\leq i\leq d$.
    For $d+1\leq i\leq N$, 
    \begin{equation}\label{eq..decop..part2..rep}
      b_i=\left\{
        \begin{array}{ll}
          3+ 3\sum\limits_{i+1\leq j\leq N,~e_{N+1\leftarrow j}\in E_{\mathrm{para}}\sqcup E_{\mathrm{fix}}}  \abs{w(e_{N+1\leftarrow j})}\abs{a_{j,i}},\quad e_{N+1\leftarrow i}\in E_\mathrm{non},\\
           3\sum\limits_{i+1\leq j\leq N,~e_{N+1\leftarrow j}\in E_{\mathrm{para}}\sqcup E_{\mathrm{fix}}}  \abs{w(e_{N+1\leftarrow j})}\abs{a_{j,i}},\quad \text{otherwise},
        \end{array}
      \right.
    \end{equation}
then inequality $b_i\geq 3\abs{a_i}$ holds for all $d+1\leq i\leq N$. Hence last line of \eqref{proofLemma...eq..tobecontinued} becomes
\begin{equation}
    \sum\limits_{1\leq i\leq d}b_i+\sum\limits_{d+1\leq i \leq N} b_i\norm{f^{i}}_\mathrm{P} \geq  \sum\limits_{1\leq i\leq d}\abs{a_i}+3\sum\limits_{d+1\leq i \leq N} \abs{a_i}\norm{f^{i}}_\mathrm{P},
\end{equation}
thus, we have 
\begin{equation}
    \sum\limits_{1\leq i\leq d}\abs{a_i}+3\sum\limits_{d+1\leq i \leq N} \abs{a_i}\norm{f^{i}}_\mathrm{P}\leq \norm{f^{N+1}}_\mathrm{P},
\end{equation}
which completes  proof of the norm inequality~\eqref{Lemma..Decomposiotion..ineq...Decompnorm} for the case $N+1$, and we finish our proof. 
\end{proof}

Next we bound the Rademacher complexity of $\fH^{N}_Q$. 
\begin{thm}\label{thm....RadmacherHQ}
    Let $\Rad_S(\fH^{N}_Q)$ be the empirical Rademacher complexity of $\fH^{N}_Q$ with respect to the  samples  $ S=\{\vx_i\}_{i=1}^n\subseteq \Omega=[0,1]^d$, then for each $N> d$, we have
    \begin{equation}\label{Thm..InEq...RadmacherBoundedbyQ}
        \Rad_S(\fH^{N}_Q)\leq 3Q\sqrt{\frac{2\log(2d)}{n}}.
    \end{equation}
\end{thm}

\begin{proof}
We shall prove \eqref{Thm..InEq...RadmacherBoundedbyQ} by induction. When $N=d+1$, by \Cref{assump..InputLinPara},  we have for any $f^{d+1}\in\fH_Q^{d+1}$ and    any sample $\vz=(z_1,z_2,\cdots,z_d)^{\T}$,  there exists coefficients $a_i$ with $1\leq i\leq d$, such that 
\begin{equation*}
    f^{d+1}(\vz) = \sum\limits_{1\leq i\leq d}  a_i z_i.
\end{equation*}
We observe that $f^{d+1}(\vz)$ is  $ f^{d+1}(\vz)=\va^\T\vz$, with $\va=(a_1,a_2,\cdots, a_d)^\T$, 
then directly from Lemma 26.11 of \cite{shalev2014understanding}, the empirical Rademacher complexity of $\fG_1=\{g\mid g(\vz)= \vb^\T\vz,~\norm{\vb}_1\leq 1 \}$ satisfies \begin{equation*}
     \Rad_S(\fG_1)\leq \max_i\norm{\vx_i}_{\infty}\sqrt{\frac{2\log(2d)}{n}}\leq \sqrt{\frac{2\log(2d)}{n}}.  
\end{equation*}
Thus, function $ \frac{1}{Q}f^{d+1}(\vz)=\frac{\va^\T\vz}{Q} \in \fG_1$, since $\frac{1}{Q}\norm{{\va}}_1=\frac{\sum_{1\leq i\leq d} \abs{a_i}}{Q}\leq 1$,
hence,\begin{equation*}
     \frac{1}{Q}\Rad_S(\fH_Q^{d+1})\leq \sqrt{\frac{2\log(2d)}{n}}, 
\end{equation*} 
and \eqref{Thm..InEq...RadmacherBoundedbyQ} holds for $N=d+1$. 

Next,  we assume that \eqref{Thm..InEq...RadmacherBoundedbyQ} holds for $d+1, d+2,\cdots, N$, and we consider the case ${N+1}$. By   definition of Rademacher complexity,
\begin{equation*}
    n\Rad_S(\fH^{N+1}_Q) 
     = \Exp_{\vtau} \sup_{f^{N+1}\in\fH^{N+1}_Q}\sum\limits_{j=1}^n \tau_j f^{N+1}(\vx_j),  
\end{equation*}    
from \eqref{Lemma..Decomposiotion..eq...DecomponAlgebra} in \Cref{lemma...Decomposition}, RHS  of the  equation reads 
\begin{align*}
&\Exp_{\vtau} \sup_{f^{N+1}\in\fH^{N+1}_Q}\sum\limits_{j=1}^n \tau_j f^{N+1}(\vx_j)    \\
&\leq \Exp_{\vtau} \sup_{(\mathrm{C1})} \sum\limits_{j=1}^n \tau_j\left(\sum\limits_{1\leq i \leq d} a_i \left(\vx_{j}\right)_i + \sum\limits_{d+1\leq i \leq N} a_i \sigma(f^i(\vx_{j}))\right) 
\end{align*}
where   condition (C1) reads
\begin{equation*}
   \mathrm{C1}:\left\{ \{a_i\}_{i=1}^N \,\middle\vert\, 
   \left( \sum\limits_{1\leq i\leq d} \abs{a_i}+ 3\sum\limits_{d+1\leq i \leq N}\abs{a_i}\norm{f^i}_\mathrm{P}\right)\leq Q\right\}.
\end{equation*}
Then, by taking out the supremum and  {positive homogeneity} of ReLU, we have
\begin{align*}
       & n\Rad_S(\fH^{N+1 }_Q) \\
        \leq& \Exp_{\vtau} \sup_{(\mathrm{C1})} \sum\limits_{j=1}^n \tau_j\left(\sum\limits_{1\leq i \leq d} a_i \left(\vx_{j}\right)_i\right)+\Exp_{\vtau} \sup_{(\mathrm{C1})} \sum\limits_{j=1}^n \tau_j\left( \sum\limits_{d+1\leq i \leq N} a_i \sigma(f^i(\vx_{j}))\right)\\
    \leq & \left(\sup_{(\mathrm{C1})} \sum\limits_{1\leq i \leq d}\abs{a_i}\right)n\Rad_S(\fG_1)+\left(\sup_{(\mathrm{C1})}\sum\limits_{d+1\leq i \leq N}\abs{a_i}\norm{f^i}_\mathrm{P}\right)\Exp_{\vtau}\sup\limits_{g^i\in\fH_1^i}\Abs{\sum\limits_{j=1}^n\tau_j\sigma(g^i(\vx_j))} 
\end{align*}
since zero function $0$ is contained in the set $\fH_{Q}^{N+1}$, then for any $\{\tau_1, \tau_2, \cdots, \tau_n\}$, it holds that
\begin{equation*}
\sup_{g^i\in\fH_1^i} \sum\limits_{j=1}^n \tau_j \sigma(g^i(\vx_{j}))\geq 0,
\end{equation*}
hence, we have
\begin{align*}
    &\sup_{g^i\in\fH_1^i} \sum\limits_{j=1}^n \tau_j \sigma(g^i(\vx_{j}))+\sup_{g^i\in\fH_1^i} \sum\limits_{j=1}^n -\tau_j \sigma(g^i(\vx_{j}))\\
   &\geq \max\left\{\sup_{g^i\in\fH_1^i} \sum\limits_{j=1}^n \tau_j \sigma(g^i(\vx_{j})),~\sup_{g^i\in\fH_1^i} \sum\limits_{j=1}^n -\tau_j \sigma(g^i(\vx_{j}))\right\}\\
   &\geq \sup_{g^i\in\fH_1^i} \Abs{\sum\limits_{j=1}^n \tau_j \sigma(g^i(\vx_{j}))},
\end{align*}
which implies that 
\begin{equation*}
     \Exp_{\vtau} \sup_{g^i\in\fH_1^i} \Abs{\sum\limits_{j=1}^n \tau_j \sigma(g^i(\vx_{j}))}\leq 2 \Exp_{\vtau} \sup_{g^i\in\fH_1^i}  {\sum\limits_{j=1}^n \tau_j \sigma(g^i(\vx_{j}))}:=2n\Rad_S\left(\vsigma \circ \fH_1^i \right),
 \end{equation*}
then directly from Lemma 26.9 of \cite{shalev2014understanding}, the empirical Rademacher complexity of $\Rad_S\left(\vsigma \circ \fH_1^i \right)$ satisfies \begin{equation*}
     \Rad_S\left(\vsigma \circ \fH_1^i \right)\leq \Rad_S\left(  \fH_1^i \right).  
\end{equation*}
From our induction hypothesis, we have that 
\begin{equation*}
    \Rad_S\left(  \fH_1^i \right)\leq 3 \sqrt{\frac{2\log(2d)}{n}},
\end{equation*}
thus 
 \begin{equation*}
     \Rad_S\left(\vsigma \circ \fH_1^i \right)\leq 3 \sqrt{\frac{2\log(2d)}{n}}.  
\end{equation*}
Finally, 
\begin{align*}
       & n\Rad_S(\fH^{N+1 }_Q) \\
    \leq&  \left(\sup_{(\mathrm{C1})} \sum\limits_{1\leq i \leq d}\abs{a_i}\right)n\Rad_S(\fG_1)+\left(\sup_{(\mathrm{C1})}\sum\limits_{d+1\leq i \leq N}\abs{a_i}\norm{f^i}_\mathrm{P}\right)\Exp_{\vtau}\sup\limits_{g^i\in\fH_1^i}\Abs{\sum\limits_{j=1}^n\tau_j\sigma(g^i(\vx_j))} \\
    \leq& \left(\sup_{(\mathrm{C1})} \sum\limits_{1\leq i \leq d}\abs{a_i}\right)\sqrt{{2n\log(2d)}}+6\left(\sup_{(\mathrm{C1})}\sum\limits_{d+1\leq i \leq N}\abs{a_i}\norm{f^i}_\mathrm{P}\right) \sqrt{{2n\log(2d)}}\\
    \leq& \sup_{(\mathrm{C1})} \left(\sum\limits_{1\leq i \leq d}\abs{a_i}+   6 \sum\limits_{d+1\leq i \leq N}\abs{a_i}\norm{f^i}_\mathrm{P}\right)\sqrt{{2n\log(2d)}}\\
    \leq&(Q+2Q)\sqrt{{2n\log(2d)}}= 3Q\sqrt{{2n\log(2d)}},
\end{align*}
which completes the proof  of~\eqref{Thm..InEq...RadmacherBoundedbyQ} for the case $N+1$, thus we finish our proof.
\end{proof}

\subsection{A Posteriori and A Priori Estimates}
We proceed to prove our main theorems,  {our proofs are extensions of the proofs in \cite{e2019prioriRes,ew2019prioriTwo}.} Firstly, we shall introduce the following theorem introduced in~\cite{shalev2014understanding}.
\begin{thm}\label{thm...BoundingGap}
Fix a hypothesis space $\fF$. Assume that for any $f\in\fF$ and $z$, $0\leq{f(z)}\leq B$, then for any $\delta>0$, with probability at least $1-\delta$ over the choice of $S=(z_1,z_2,\cdots,z_n)$, we have for any function $f(\cdot)$,
\begin{equation}\label{thm.. eq...BoundingGap}
    \Abs{\frac{1}{n} \sum_{i=1}^n f(z_i)-\Exp_z{f(z)}}\leq 2\Exp_{S'}\Rad_{S'}(\fF)+B\sqrt{\frac{\log(2/\delta)}{2n}}.
\end{equation}
\end{thm}

\begin{proof}[Proof of  \Cref{thm..Aposteriori}]
    Let $\fF_Q:=\{\ell(\cdot;\vtheta)\mid\norm{\vtheta}_\mathrm{P}\leq Q\}$. Note that $\ell(\cdot;\vtheta)$ is a $1$-Lipschitz function and bounded by $0$ and $\frac{1}{2}$, then directly from Lemma 26.9 of \cite{shalev2014understanding},
    
     {$$\Rad_S(\fF_Q)= \Rad_S\left(\ell \circ \fH_Q\right)
  \leq \Rad_S(\fH_Q)\leq 3Q\sqrt{\frac{2\log(2d)}{n}},$$ }
from  the above inequality, combined  with   \Cref{thm....RadmacherHQ} and \Cref{thm...BoundingGap} leads to  {the following inequalities with probability at least $1-\delta_Q$}
   {
  \begin{align*}
      \sup_{\norm{\vtheta}_\mathrm{P}\leq Q}\abs{\RD(\vtheta)-\RS(\vtheta)}
      &\leq 2\Exp_{S'}\Rad_{S'}(\fF_Q)+\frac{1}{2}\sqrt{\frac{\log(2/\delta_Q)}{2n}}\\
      &\leq 6Q\sqrt{\frac{2\log(2d)}{n}}+\frac{1}{2}\sqrt{\frac{\log(2/\delta_Q)}{2n}}.
  \end{align*}
  }
   {We use this bound with $Q=1,2,\ldots$ and $\delta_Q=\frac{6\delta}{\pi^2Q^2}$. Note that $1-\sum_{Q=1}^\infty\delta_Q=1-\delta$ and consider the union bound. Then with probability at least $1-\delta$, the following inequality holds for all $Q>0$,}
  \begin{equation*}
      \sup_{\norm{\vtheta}_\mathrm{P}\leq Q}\abs{\RD(\vtheta)-\RS(\vtheta)}
      \leq 6Q\sqrt{\frac{2\log(2d)}{n}}+\frac{1}{2}\sqrt{\frac{\log(\pi^2Q^2/3\delta)}{2n}}.
  \end{equation*}
   {To further use this inequality, we let $Q$ be the integer part of $\norm{\vtheta}_\mathrm{P}$.}
  Note that   $Q\leq \norm{\vtheta}_\mathrm{P}+1$. Therefore, we have
  \begin{align*}
      \abs{\RD(\vtheta)-\RS(\vtheta)}
      &\leq 6(\norm{\vtheta}_\mathrm{P}+1)\sqrt{\frac{2\log(2d)}{n}}+\frac{1}{2}\sqrt{\frac{1}{2n}\log\frac{\pi^2}{3\delta}+\frac{\log (\norm{\vtheta}_\mathrm{P}+1)^2}{2n}}\\
      &\leq 6(\norm{\vtheta}_\mathrm{P}+1)\sqrt{\frac{2\log(2d)}{n}}+\frac{1}{2}\frac{\norm{\vtheta}_\mathrm{P}+1}{\sqrt{2n}}+\frac{1}{2}\sqrt{\frac{\log(\pi^2/3\delta)}{2n}},
  \end{align*}
  where in the last inequality we used the fact that $\log(a)\leq a$, for $a\geq 1$, and  $\sqrt{a+b}\leq\sqrt{a}+\sqrt{b}$ for positive $a$ and $b$.
\end{proof}

\begin{proof}[Proof of Theorem \ref{thm..Apriori}]
    Firstly, the population risk can be decomposed into 
    \begin{equation*}
        \RD(\vtheta_{S,\lambda})
        =\RD(\tilde{\vtheta})+[\RD(\vtheta_{S,\lambda})-J_{S,\lambda}(\vtheta_{S,\lambda})]
        +[J_{S,\lambda}(\vtheta_{S,\lambda})-J_{S,\lambda}(\tilde{\vtheta})]
        +[J_{S,\lambda}(\tilde{\vtheta})-\RD(\tilde{\vtheta})].
    \end{equation*}
   Set $\lambda=3\lambda_0\sqrt{{2\log(2d)}}$, where $\lambda_0\geq 2+\frac{1}{12\sqrt{\log (2d)}}$, then by Theorem \ref{thm..Aposteriori}, we have with probability at least $1-\delta/2$,
   \begin{equation}\label{eq..ProofAprioriDenseNetPart1}
    \begin{aligned}
        \RD(\vtheta_{S,\lambda})-J_{S,\lambda}(\vtheta_{S,\lambda})
        &= \RD(\vtheta_{S,\lambda})-\RS(\vtheta_{S,\lambda})-3\lambda_0\sqrt{\frac{2\log(2d)}{n}}\norm{\vtheta_{S,\lambda}}_\mathrm{P} \\
        &\leq (\norm{\vtheta_{S,\lambda}}_\mathrm{P}+1)\frac{3(2-\lambda_0)\sqrt{2\log(2d)}+\frac{1}{2\sqrt{2}}}{\sqrt{n}}\\&~~+3\lambda_0\sqrt{\frac{2\log(2d)}{n}}+\frac{1}{2}\sqrt{\frac{\log(2\pi^2/3\delta)}{2n}} \\
        &\leq 3\lambda_0\sqrt{\frac{2\log(2d)}{n}}+\frac{1}{2}\sqrt{\frac{\log(2\pi^2/3\delta)}{2n}},
    \end{aligned}\end{equation}
    where in the last inequality we used the fact that $2+\frac{1}{12\sqrt{\log (2d)}}= 2+\frac{1}{2\sqrt{2}\cdot 3\cdot\sqrt{2\log(2d)}}$.
    
    Next, by Theorem \ref{thm..Aposteriori} again, we have with probability at least $1-\delta/2$,
    \begin{equation}\label{eq..ProofAprioriDenseNetPart2}
    \begin{aligned}
        J_{S,\lambda}(\tilde{\vtheta})-\RD(\tilde{\vtheta})
        &= \RS(\tilde{\vtheta})-\RD(\tilde{\vtheta})+3\lambda_0\sqrt{\frac{2\log(2d)}{n}}\norm{\tilde{\vtheta}}_\mathrm{P} \\
        &\leq (\norm{\tilde{\vtheta}}_\mathrm{P}+1)\frac{3(2+\lambda_0)\sqrt{2\log(2d)}
        +\frac{1}{2\sqrt{2}}}{\sqrt{n}}\\&~~-3\lambda_0\sqrt{\frac{2\log(2d)}{n}}+\frac{1}{2}\sqrt{\frac{\log(2\pi^2/3\delta)}{2n}}.
    \end{aligned}
    \end{equation}
    Finally, we observe  that $J_{S,\lambda}(\vtheta_{S,\lambda})-J_{S,\lambda}(\tilde{\vtheta})\leq 0$ by   optimality of   $\vtheta_{S,\lambda}$, then combined with  the fact $\norm{\tilde{\vtheta}}_\mathrm{P}\leq 6\norm{f^*}_{\fB}$ from Theorem \ref{thm..ApproxError},
    our proof   is finished by collecting altogether \eqref{eq..ProofAprioriDenseNetPart1}, \eqref{eq..ProofAprioriDenseNetPart2}    and  $\RD(\tilde{\vtheta})\leq \frac{3\norm{f^*}_{\fB}^2}{2N_\mathrm{non}}$ from Theorem \ref{thm..ApproxError}, i.e.,
    \begin{align*}
        & \RD(\vtheta_{S,\lambda})\\
        =&\RD(\tilde{\vtheta})+[\RD(\vtheta_{S,\lambda})-J_{S,\lambda}(\vtheta_{S,\lambda})]
        +[J_{S,\lambda}(\vtheta_{S,\lambda})-J_{S,\lambda}(\tilde{\vtheta})]
        +[J_{S,\lambda}(\tilde{\vtheta})-\RD(\tilde{\vtheta})]\\
 \leq &\frac{3\norm{f^*}_{\fB}^2}{2N_\mathrm{non}}+3\lambda_0\sqrt{\frac{2\log(2d)}{n}}+\frac{1}{2}\sqrt{\frac{\log(2\pi^2/3\delta)}{2n}}\\
 &+(\norm{\tilde{\vtheta}}_\mathrm{P}+1)\frac{3(2+\lambda_0)\sqrt{2\log(2d)}
        +\frac{1}{2\sqrt{2}}}{\sqrt{n}}-3\lambda_0\sqrt{\frac{2\log(2d)}{n}}+\frac{1}{2}\sqrt{\frac{\log(2\pi^2/3\delta)}{2n}}\\
        \leq& \frac{3\norm{f^*}_{\fB}^2}{2N_\mathrm{non}}+(6\norm{f^*}_{\fB}+1)\frac{3(2+\lambda_0)\sqrt{2\log(2d)}
        +\frac{1}{2\sqrt{2}}}{\sqrt{n}}+\sqrt{\frac{\log(2\pi^2/3\delta)}{2n}}\\
        \lesssim& \frac{\norm{f^*}_{\fB}^2}{N_\mathrm{non}}+  \frac{1}{\sqrt{n}}
        \left(\lambda(\norm{f^*}_{\fB}+1)+\sqrt{\log 1/\delta}\right),
    \end{align*}
    which finishes our proof.
\end{proof}

\subsection{Applications to DenseNet}\label{subsection....DenseNet}
We  directly apply our results, especially \Cref{thm..Apriori}, to obtain the a priori estiamte for DenseNet.
\begin{cor}[A priori estimate for DenseNet]\label{cor..Apriori...DenseNet}
  Suppose  $f^*\in\fB$, $\lambda= \Omega(\sqrt{\log d})$, and  assume that  $\vtheta_{S,\lambda}$ is an optimal solution for the regularaized model \eqref{Subsection...MainREsults...eq..RegularizedRisk}, i.e., $\vtheta_{S,\lambda}\in\arg\min_{\vtheta}J_{S,\lambda}(\vtheta)$,
  then for any $\delta\in(0,1)$, with probability at least $1-\delta$ over the choice of the training sample $S$,  we have for any DenseNet $f(\cdot;\vtheta)$, its population risk satisfies
  \begin{equation}\label{eq..AprioriEstimate..DenseNet} 
    \begin{aligned}
        \RD(\vtheta_{S,\lambda})
        &:=\Exp_{\vx\sim\fD}\tfrac{1}{2}(f(\vx;\vtheta_{S,\lambda})-f^*(\vx))^2\\
        &\lesssim \frac{\norm{f^*}_{\fB}^2}{L^2m}+  \frac{1}{\sqrt{n}}
        \left(\lambda(\norm{f^*}_{\fB}+1)+\sqrt{\log 1/\delta}\right).
    \end{aligned}
  \end{equation}
\end{cor}
\begin{proof}
    {  In the case of DenseNet, as is shown already in \Cref{ex...Densenet}, the matrix representation $\mA$ incorporates the form
\begin{equation*}
        \begin{pmatrix}
            \mzero & & & & & & &\\
            \mV & \mzero & & & & & &\\
            & \mB^{[1]} & \mzero & & & & &\\
            & & \mB^{[2]} & \mzero & & & &\\
            & & & \ddots & \ddots & & &\\
            & & & & \mB^{[l]} & \mzero & &\\
            & & & & & \ddots & \ddots &\\
            & & & & & & \mB^{[L]} &\mzero &\\
            & & & & & & & \vu^\T &\mzero
        \end{pmatrix},
    \end{equation*}}
    with    $\mV$ taking size  $d_0\times d,~d_0=k_0$, length of  vector $\vu$ being  $d_L,~d_L=k_0+Lk$, and for each matrix block $\mB^{[l]},~l=1,\cdots,L:$
    \begin{equation*}
        \mB^{[l]}=
        \begin{pmatrix}
            \mW^{[l]}_{} &  &\\
             & \sigma\mI_{p_l\times p_l} & \\
           \bar{\mI}_{(k_0+lk)\times (k_0+(l-1)k)}  &  & \bar{\mU}^{[l]}
        \end{pmatrix},
    \end{equation*}
    where  $\mB^{[l]}$ has size  $(2p_{l}+d_{l})\times(2p_{l}+d_{l-1}),~p_l=lm,~d_l=k_0+lk$,  and the size of its components $\left\{\mW^{[l]},\sigma\mI_{p_l\times p_l}, \bar{\mI}_{(k_0+lk)\times (k_0+(l-1)k)} ,  \bar{\mU}^{[l]}\right\}$ respectively  reads $p_l\times d_{l-1}$, $p_l\times p_l$, $d_{l}\times d_{l-1}$, and $d_{l}\times p_l$. 
 Specifically,  $\bar{\mI}_{(k_0+lk)\times (k_0+(l-1)k)} =\begin{pmatrix}
      \mI_{d_{l-1}\times d_{l-1}}\\\mzero_{(d_{l}-d_{l-1})\times d_{l-1}}
    \end{pmatrix}$. Therefore, DenseNet satisfies \Cref{assump..BlockMatrix}. Moreover, we have
\begin{equation}\label{Eq...NumberDense}
    N_{\mathrm{non}} =\sum_{ l=1}^L p_l=\sum_{l=1}^L lm=\frac{L(L+1)}{2}m,
\end{equation}
 apply \Cref{thm..Apriori} directly,  \eqref{eq..AprioriEstimate} reads exactly as \eqref{eq..AprioriEstimate..DenseNet}. 
\end{proof}
\section{Conclusion}\label{Section...Conclusion}
Our main contribution is the introduction of a novel representation  of feedforward neural networks, namely the nonlinear weighted DAG. 
This representation provides further understanding of efficiency of   shortcut connections utilized by various networks.  
We also show in detail how  typical examples of feedforward neural networks can be represented using this  formulation.

Moreover, we derive  a priori estimates in avoidance of the CoD for   neural networks satisfying the assumption of shortcut connections~
(\Cref{assump..BlockMatrix}). Our estimates serve as an extension for the results in~\cite{e2019prioriRes,ew2019prioriTwo}, and  key to our analysis is the  employment
of   weighted path norm. As  demonstrated in~\cite{e2019prioriRes,ew2019prioriTwo}, the weight path norm is capable of  bounding the approximation and estimation errors simultaneously for ResNet and two-layer network, and we show that it is also the case for DenseNet.

\section*{Acknowledgments}
We would like to give special thanks to
Prof. Weinan E for his helpful discussions.  {We would also like to thank anonymous referees for numerous comments that helped to improve previous drafts of this paper.} This work is also sponsored by the National Natural Science Foundation of China Grant No. 12101401 (T. L.) and Shanghai Municipal of Science and Technology Major Project NO.2021SHZDZX0102 (T.L.).
\bibliographystyle{abbrv}
\bibliography{DenseNet}
\appendix
\section{Properties of the Representation    and  Symbol}\label{appendix}
\begin{proof}[Proof of \Cref{prop..property}]
    1. As a lower triangular matrix, $\mA(\vtheta,\xi)$ is obviously nilpotent.
    
    2. By definition, we have $$\vz_{s+1}-\vz_s=\mA(\vtheta,\sigma)\vz_s-\mA(\vtheta,\sigma)\vz_{s-1}.$$ Since $\sigma(\cdot)$ is $1$-Lipschitz, then
    \begin{align*}
        \Abs{(\mA\vz_s-\mA\vz_{s-1})_i}
       &= \Abs{\sum_{j=1}^N\left[(\mA)_{ij}(\vz_s)_j-(\mA)_{ij}(\vz_{s-1})_j\right]}\\
        &\leq \sum_{j=1}^N\left(\mA(\abs{\vtheta},1)\right)_{ij}\abs{(\vz_s-\vz_{s-1})_j}.
    \end{align*}
    Thus for all $i=1,\cdots,N$, ~$(\abs{\vz_{s+1}-\vz_s})_i\leq (\mA(\abs{\vtheta},1)\abs{\vz_s-\vz_{s-1}})_i$. Therefore, inductively we have
    \begin{equation*}
        \norm{\vz_{s+1}-\vz_s}_\infty
        \leq \norm{\mA(\abs{\vtheta},1)\abs{\vz_s-\vz_{s-1}}}_\infty\leq \cdots
        \leq  \norm{\mA^s(\abs{\vtheta},1)\abs{\vz_1-\vz_0}}_\infty.
    \end{equation*}
    
    3. Since $\mA(\abs{\vtheta},1)$ is nilpotent, for sufficiently large $s$, we have $$\norm{\vz_{s+1}-\vz_s}_\infty\leq  \norm{\mzero_{N\times N}\abs{\vz_1-\vz_0}}_\infty=0.$$ Hence $\{\vz_s\}_{s\geq 0}$ is a Cauchy sequence, and its limit exists.
    
    4. This is straightforward by  definition of the sequence $\{\vz_s\}_{s\geq 0}$ as well as the existence of the limit $\vz_\infty$.
    
    5. Since $\vz_0=\mP_0\vz_s$ for any $s$, we have $\vz_s=\mP_0\vz_{s-1}+\mA\vz_{s-1}=\bar{\mA}\vz_{s-1}$. Hence, it holds naturally that $\vz_\infty=\bar{\mA}^\infty\vz_0$.
 
    6.  A simple calculation is sufficient to show that $\norm{\vtheta}_\mathrm{P}=\vone_\mathrm{out}^\T\sum_{s=0}^\infty\mA^s(\abs{\vtheta},3)\vone_\mathrm{in}$,
    and by taking advantage of  the  property of nilpotency, we have that $$\sum_{s=0}^\infty\mA^s(\abs{\vtheta},3)=(\mI_{N\times N}-\mA(\abs{\vtheta},3))^{-1},$$ then
    $\norm{\vtheta}_\mathrm{P}=\vone_\mathrm{out}^\T(\mI_{N\times N}-\mA(\abs{\vtheta},3))^{-1}\vone_\mathrm{in}$. Moreover, by  definition of $\mP_0$, we observe that
    $\norm{\vtheta}_\mathrm{P}=\vone_\mathrm{out}^\T\bar{\mA}^\infty(\abs{\vtheta},3)\vone_\mathrm{in}$.
    
        7. These are straightforward by definition of the entries in the  adjacency matrix representation.

\end{proof}
\end{document}